\documentclass[a4paper]{article}
\usepackage[a4paper,top=3.5cm,bottom=3cm,left=3.5cm,right=3.5cm,marginparwidth=1.5cm]{geometry}
\usepackage{amstext}
\usepackage{amssymb}
\usepackage{amsmath}
\usepackage{amsthm}

\usepackage{mathtools}

\usepackage[utf8]{inputenc}
\usepackage{bbm}
\usepackage{times}
\usepackage{graphicx,color}
\usepackage{array,float}
\usepackage{url}

\usepackage{paralist}
\usepackage{hyperref}

\usepackage{algorithm}
\usepackage{xparse,etoolbox}
\usepackage{threeparttable}
\usepackage{enumitem}
\usepackage{dsfont}

\usepackage[mathscr]{euscript}

\usepackage{lipsum,wrapfig}
\usepackage{multicol}
\usepackage[font=small,labelfont=bf]{caption}
\usepackage{verbatim}

\usepackage{latexsym}

\usepackage{multirow}
\usepackage{thmtools}
\usepackage{thm-restate}
\usepackage{algorithmic}

\usepackage{MnSymbol}
\DeclareMathAlphabet\mathbb{U}{msb}{m}{n}
\usepackage{xpatch}

\usepackage{MnSymbol}
\DeclareMathAlphabet\mathbb{U}{msb}{m}{n}
\usepackage{xpatch}

\def\Rset{\mathbb{R}}

\let\P\undefined

\DeclareMathOperator*{\P}{\mathbb{P}}
\DeclareMathOperator*{\E}{\mathbb E}

\DeclareMathOperator*{\argmin}{argmin}

\DeclareMathOperator{\fat}{fat}

\DeclareMathOperator*{\sgn}{sgn}

\DeclarePairedDelimiter{\bracket}{[}{]}
\DeclarePairedDelimiter{\curl}{\{}{\}}
\DeclarePairedDelimiter{\norm}{\|}{\|}
\DeclarePairedDelimiter{\paren}{(}{)}
\DeclarePairedDelimiter{\tri}{\langle}{\rangle}

\newcommand{\cA}{\mathcal{A}}

\newcommand{\cX}{\ensuremath{\mathcal{X}}}
\newcommand{\cY}{\ensuremath{\mathcal{Y}}}
\newcommand{\cZ}{\ensuremath{\mathcal{Z}}}	
\newcommand{\cD}{\mathcal{D}}

\newcommand{\cH}{\mathcal{H}}

\newcommand{\cV}{\mathcal{V}}
\newcommand{\cC}{\mathcal{C}}

\newcommand{\cN}{\mathcal{N}}

\newcommand{\cW}{\mathcal{W}}

\newcommand{\re}{\mathbb{R}}

\newcommand{\hprv}{h^{\mathsf{priv}}}

\newcommand{\B}{\mathbb{B}}

\newcommand{\lin}{\mathsf{Lin}}
\newcommand{\eff}{\mathsf{Eff}}
\newcommand{\prv}{\mathsf{Priv}}
\newcommand{\mrg}{\mathsf{Mrg}}
\newcommand{\mrgn}{\mathsf{Mrg}}
\newcommand{\krn}{\mathsf{Ker}}

\newcommand{\pr}[2]{\underset{#1}{\mathbb{P}}\left[ #2 \right]}
\newcommand{\ex}[2]{\underset{#1}{\mathbb{E}}\left[ #2 \right]}

\newcommand{\ind}{\mathbf{1}}

\newcommand{\emperr}{\widehat{\mathsf{err}}}

\newcommand{\basichng}{\ell}
\newcommand{\emphng}{\widehat{L}}
\newcommand{\genericloss}{g}
\newcommand{\hinge}{\ell^{\mathsf{\,hinge}}}
\newcommand{\fatshatteringbound}{M}
\newcommand{\hH}{\widehat{\mathcal{H}}_{\rho}}

\newcommand{\hmax}{\text{h}_{\max}}

\newcommand{\labmarg}{\mathsf{LabMarg}}
\newcommand{\opt}{\mathsf{opt}}

\newcommand{\hpsi}{\widehat{\psi}}

\usepackage[mathscr]{euscript}
\newcommand{\sD}{\mathscr{D}}

\newcommand{\sC}{{\mathscr C}}
\newcommand{\sF}{{\mathscr F}}
\newcommand{\sH}{{\mathscr H}}

\newcommand{\h}{\widehat}

\newcommand{\wt}{\widetilde}

\newcommand{\ignore}[1]{}

\declaretheorem[name=Theorem,numberwithin=section]{thm}
\newtheorem{lem}{Lemma}[section]

\newtheorem{defn}[lem]{Definition}
\newtheorem{fact}[lem]{Fact}
\newtheorem{claim}[lem]{Claim}

\hypersetup{
  colorlinks   = true,
  urlcolor     = blue,
  linkcolor    = blue,
  citecolor   = blue
}

\definecolor{BUred}{rgb}{0.85, 0.0, 0.05}
\definecolor{azure}{rgb}{0.0, 0.0, 0.85}
\definecolor{cobalt}{rgb}{0.0, 0.15, 0.75}

\newif\ifnotes
\notesfalse

\ifnotes

\newcommand{\rnote}[1]{ [\textcolor{blue}{Raef: #1}] }
\newcommand{\mnote}[1]{ [\textcolor{red}{Mehryar: #1}] }
\newcommand{\thnote}[1]{ [\textcolor{purple}{Theertha: #1}] }

\else

\newcommand{\rnote}[1]{}
\newcommand{\mnote}[1]{}
\newcommand{\thnote}[1]{}
\fi

\title{Differentially Private Learning with Margin Guarantees}

\author{%
	Raef Bassily\thanks{Google Research \& Department of Computer Science and Engineering, The Ohio State University. \texttt{bassily.1@osu.edu}}
 \and Mehryar Mohri\thanks{Google Research \& Courant Institute of Mathematical Sciences, 
New York, NY. \texttt{mohri@google.com}}
 \and Ananda Theertha Suresh\thanks{Google Research, NY. \texttt{theertha@google.com}}
}
\date{}

\usepackage[toc, page, header]{appendix}
\begin{document}
\maketitle
\begin{abstract}%
  We present a series of new differentially private (DP) algorithms
  with dimension-independent margin guarantees. For the family of
  linear hypotheses, we give a pure DP learning algorithm that
  benefits from relative deviation margin guarantees, as well as an
  efficient DP learning algorithm with margin guarantees.  We also
  present a new efficient DP learning algorithm with margin guarantees
  for kernel-based hypotheses with shift-invariant kernels, such as
  Gaussian kernels, and point out how our results can be extended to
  other kernels using oblivious sketching techniques.  We further give
  a pure DP learning algorithm for a family of feed-forward neural
  networks for which we prove margin guarantees that are independent
  of the input dimension.  Additionally, we describe a general label DP
  learning algorithm, which benefits from relative deviation margin
  bounds and is applicable to a broad family of hypothesis sets,
  including that of neural networks. Finally, we show how our DP
  learning algorithms can be augmented in a general way to include
  model selection, to select the best confidence margin parameter.
\end{abstract}


\section{Introduction}

Preserving privacy is a crucial objective for machine learning
algorithms. A widely adopted criterion in statistical data privacy is
the notion of differential privacy (DP)
\cite{DworkMcSherryNissimSmith2006,Dwork2006,DworkRoth2014}, which
ensures that the information gained by an adversary is roughly
invariant to the presence or absence of an individual in a dataset.
Despite the remarkable theoretical and algorithmic progress in
differential privacy over the last decade or more, however, its
application to learning still faces several obstacles.
A recent series of publications have shown that differentially private
PAC learning of infinite hypothesis sets is not possible, even for
common hypothesis sets such as that of linear functions.  In fact,
this is the case for any hypothesis set containing threshold functions
\cite{BunNissimSalilVenkatesan2015,
  AlonLivniMalliarisMoran2019}. These results imply serious
limitations for private agnostic learnability\ignore{ and motivate
seeking alternative guarantees for learning standard hypothesis sets}.

Another rich body of literature has studied differentially private
empirical risk minimization (DP-ERM) and differentially private
stochastic convex optimization (DP-SCO) (e.g.,
\cite{ChaudhuriMonteleoniSarwate2011,JTOpt13,
  BassilySmithThakurta2014, bassily2019private,
  FeldmanKorenTalwar2020, SSTT:21, bassily2021non, asi2021private, bassily2021differentially}).
When the underlying optimization problem is constrained
(\emph{constrained setting}), tight upper and lower bounds have been
derived for the excess empirical risk of DP-ERM
\cite{BassilySmithThakurta2014} and for the excess population risk
for DP-SCO \cite{bassily2019private, FeldmanKorenTalwar2020}.  These
results show that learning guarantees necessarily admit a dependency
on the dimension $d$ of the form $\sqrt{d}/m$, where $m$ is the sample
size. This dependency is persistent, even in the special case of
\emph{generalized linear losses} (GLLs)
\cite{BassilySmithThakurta2014}, which limits the benefit of such
guarantees, since learning algorithms typically deal with
high-dimensional spaces.

When the underlying optimization problem is unconstrained
(\emph{unconstrained setting}) and the loss is a generalized linear
loss, the bounds given by \cite{JTOpt13}, \cite{SSTT:21} and
\cite{bassily2021differentially} are dimension-independent but they admit a dependency
on $\norm{w^\ast}^2$, where $w^\ast$ is the unconstrained minimizer of
the expected loss (population risk), or $\norm{\widehat{w}}^2$, where
$\widehat{w}$ is the unconstrained minimizer of the empirical
loss. Since the problem is unconstrained, the norm of these vectors
can be very large, even for classification problems for which the
minimizer of the zero-one loss admits a relatively small norm.
Thus, in both the constrained and unconstrained settings, the learning
guarantees derived from DP-ERM and DP-SCO are weak for hypothesis sets
commonly used in machine learning.

The results just mentioned raise some fundamental questions about
private learning: is differentially private learning with favorable
(dimension-independent) guarantees possible for standard hypothesis
sets? Must one resort to distribution-dependent bounds
instead?\ignore{ Do we need to restrict ourselves to
  distribution-dependent guarantees?} In view of the negative
PAC-learning results and other learning bounds mentioned earlier, we
will seek instead optimistic margin-based learning bounds.

In the context of classification, learning bounds for linear
hypotheses based on the dimension or, more generally, based on the
VC-dimension of the hypothesis set are known to be too pessimistic
since they deal with the worst case. Instead, margin bounds have been
shown to be the most informative and useful guarantees
\cite{KoltchinskiiPanchenko2002,SchapireFreundBartlettLee97\ignore{,
    mohri2018foundations,cortes2021relative}}. This motivates our
study of differentially private learning algorithms with margin-based
guarantees. Note that our \emph{confidence-margin} analysis and
guarantees do not require the hard-margin separability assumptions
adopted in \cite{BlumDworkMcSherryNissim2005,le2020efficient}, which
is a strong assumption that typically does not hold in
practice. Another existing study that deals with somewhat related
questions is that of \cite{ChaudhuriHsuSong2014}. But, the paper deals
with a specific class of maximization problems and adopts a
non-standard definition of margin.
Another related line of work is that of 
\cite{rubinstein2009learning}
and \cite{ChaudhuriMonteleoniSarwate2011}
on DP Kernel classifiers, which we discuss
in detail in Section~\ref{sec:related}. 
\ignore{
Note also that one way to circumvent the negative results discussed
above is to leverage unlabeled public data, which provides key
information about the input distribution, thereby resulting in more
favorable guarantees \cite{BeimelNissimStemmer2013,bassily2018model,
  BassilyMoranAlon2019}. However, public data, when available, can
also be used to derive more favorable margin guarantees. In general,
we may not always have at our disposal large amounts of public data
for the same task and from the same distribution.
}

\textbf{Main contributions.}
We present a series of new differentially private (DP) algorithms with
dimension-independent margin guarantees. In Section~\ref{sec:linear},
we study the family of linear hypotheses. We first give a pure DP
learning algorithm with relative deviation margin guarantees that is
computationally inefficient. Next, we present an efficient DP learning
algorithm with margin guarantees.
In Section~\ref{sec:kernels}, we consider kernel-based hypothesis
sets. We present a new efficient DP learning algorithm with margin
guarantees for such hypothesis sets, assuming that the positive
definite kernel used is shift-invariant, as with the most commonly
used Gaussian kernels. We further briefly discuss how recent kernels
approximation results using oblivious sketching can be used to extend
our results to other kernel functions, including polynomial kernels
and many other kernels that can be approximated using polynomial
kernels, as well and the neural tangent kernel (NTK) or arc-cosine
kernels. In Section~\ref{sec:NNs}, we initiate the study of DP
learning of neural networks with margin guarantees.  We design a pure
DP learning algorithm for a family of feed-forward neural networks for
which we prove margin guarantees that are independent of input
dimension.
In Section~\ref{sec:label}, we further present a \emph{label privacy}
learning algorithm, which we show benefits from relative deviation
margin bounds. The algorithm and its guarantee are applicable to a
broad family of hypothesis sets, including that of neural networks. In
Appendix~\ref{app:rho}, we show how our DP learning algorithms can be
augmented in a general way to include model selection, to select the
best confidence margin parameter.

Let us emphasize some key novelty of our work: (i) we give algorithms
and more favorable learning guarantees than previous work on
unconstrained GLLs; (ii) our guarantees and analyses are expressed in
terms of the confidence-margin, in contrast with the geometric margin,
which relies on a strong separability assumption; (iii) we give a more
efficient algorithm for linear classification based on a faster
construction for the JL-transform and faster DP-ERM algorithm; (iv) we
present new \emph{private} learning bounds for DP kernel classifiers
that are nearly the same as the standard \emph{non-private} bounds,
without resorting to the strong assumptions adopted in prior work; (v) we initiate the study of DP learning of neural nets with margin guarantees and derive the first margin bound for this problem that has no dependence on the input dimension and better dependence on the network parameters than the bounds attained via uniform convergence.

We also wish to highlight the novelty of our analysis: (i) while the
general structure of our algorithms for linear classifiers is
superficially similar to that of \cite{le2020efficient}, our results
require a new analysis that takes into account the scale-sensitive
nature of the margin loss and the $\rho$-hinge loss\ignore{ (defined
  in Section~blah)}; (ii) our solutions include new algorithmic ideas
and analyses for DP kernel classifiers, such as the use of
JL-transform and a new analysis that uses regularized ERM as a
reference; (iii) our margin bound for DP neural nets entails a new analysis of embedding-based ``network compression'' technique.  It is also important to point out that, even though we use
the $\rho$-hinge loss in our efficient constructions for linear and kernel classifiers, our results can
be easily extended to other convex surrogates for the zero-one loss,
such as the logistic loss.

\subsection{Related work}
\label{sec:related}

In this section, we discuss in more detail other previous studies that are
the most directly related to the work we present.

\textbf{Prior work on unconstrained GLLs.} \cite{JTOpt13} and
\cite{SSTT:21} showed that it is possible to derive
dimension-independent risk bounds for DP-ERM and DP-SCO in the context
of linear prediction, when the parameter space is unconstrained and
the loss function is convex and Lipschitz (GLL). However, their bounds
scale with $\norm{w^\ast}$, the norm of the optimal unconstrained
minimizer of the expected loss of a surrogate loss such as the hinge
loss. Also, using their techniques for unconstrained DP-ERM for GLLs
together with uniform convergence would yield generalization error
bounds that scale with the norm of the unconstrained empirical risk
minimizer $\h w$.

The first issue with this line of work is that the norms of such
unconstrained solutions can be very large, thereby resulting in
uninformative bounds. In fact, one can construct simple,
low-dimensional examples, where $\norm{w^\ast} = \Omega(m)$ while
there is a predictor $w$ with $\norm{w} = O(1)$ that attains the same
expected zero-one error, see Figure~\ref{fig:example}.  A detailed
analysis of that example in presented in
Appendix~\ref{app:example}. Moreover, these studies assume prior
knowledge of $\norm{w^\ast}$, which is not a realistic
assumption. Readily applying the techniques of these studies without
assuming that prior knowledge results in an explicit dependence on $(1
+ \norm{w^\ast}^2)$, which is even less favorable. Furthermore,
estimating this norm privately in the unconstrained setting requires a
non-trivial construction and argument. Perhaps more importantly, the
paradigm adopted in this line of work is to first devise an algorithm
and next derive bounds for its excess risk. In contrast, we start from
strong generalization error bounds, which we use to guide the design
of our algorithm\ignore{: our algorithms are designed to ``optimize''
  the margin bounds we derive}.

\begin{figure*}[t]
\centering
\includegraphics[scale=0.2]{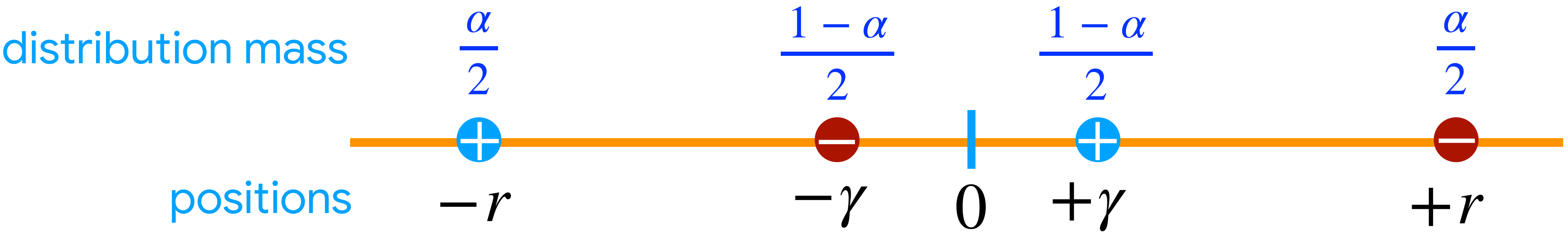}
\caption{Simple example in dimension one for which the minimizer of
  the expected hinge loss $\E[\hinge(w)]$ is $w^* = \frac{1}{\gamma}$
  and thus $\norm{w^*} = \frac{1}{\gamma} \gg 1$ for $\gamma \ll
  1$. Here, any other $w > 0$, in particular with a small norm,
  achieves the same zero-one loss as $w^*$.}
\label{fig:example}
\vskip -.1in
\end{figure*}

\textbf{Prior work on DP learning of hard-margin halfspaces.}
\cite{BlumDworkMcSherryNissim2005} and \cite{le2020efficient}
studied DP learning of linear classifiers in the separable setting,
that is with a hard- or \emph{geometric
margin}. \cite{BlumDworkMcSherryNissim2005} gave a construction based
on a private version of the Perceptron algorithm, which results in a
dimension-independent bound on the expected error. This result was
later improved by \cite{le2020efficient} who gave new constructions
with dimension-independent guarantees based on embeddings, namely, the
Johnson-Lindenstrauss (JL) transform. Note that the hard-margin
separability is a strong assumption that typically does not hold in
practice. Moreover, the construction suggested by the authors requires
knowledge of the margin for their guarantees to be valid. In contrast,
our work considers the more general notion of \emph{confidence
margin}, which does not require the existence of a geometric margin
and applies to realistic scenarios with non-separable data. Moreover,
the confidence-margin parameter, $\rho$, in our algorithms is tunable
and can be optimized (which we do in
Appendix~\ref{app:rho}). Importantly, our algorithms still yields
meaningful learning guarantees even if this parameter is not
optimized. Our algorithms for linear classifiers also makes use of an
embedding as pre-processing step. However despite a structure similar
on the surface to that of \cite{le2020efficient}, our algorithm
requires a new analysis together as well as a precise setting of the
parameters. This includes a careful analysis to deal with the
scale-sensitive nature of our bounds, due to the absence of a
hard-margin. This is evident in the setting of the embedding
parameters in our efficient algorithm based on the hinge loss, which
is different from that in \cite{le2020efficient}. In particular, in
our setting, we choose the embedding dimension to be approximately $m$
to control the impact of the embedding approximation error on
the empirical hinge loss when a hard margin is absent or unknown.

\textbf{Prior work on DP Kernel classifiers.}
\cite{rubinstein2009learning} were the first to provide differentially
private constructions for SVMs in both the finite-dimensional feature
space and kernel settings. However, their constructions are suboptimal
and the resulting bounds suffer from a polynomial dependence on the
dimension of the feature space. In particular, they consider SVM
learning with a shift-invariant kernel. Their algorithm in this case
is based on defining a finite-dimensional approximation of the kernel
using the technique of \cite{RahimiRecht2007}, which reduces the
problem to a finite-dimensional SVM learning problem. Their private
solution is based on perturbing the non-private finite-dimensional
predictor. In addition to the polynomial dependence on the input
dimension, their error bound in this case is sub-optimal in its
dependence on the sample size $m$ and admits an explicit dependence on
the on the $L_1$-norm of the dual variables of the SVM
\cite[Theorem~14]{rubinstein2009learning}. In general, this norm can
be as large as $\sqrt{m}$ and, in such cases, their error bound
becomes vacuous.

\cite{ChaudhuriMonteleoniSarwate2011} gave a similar construction for
shift-invariant kernels. However, their error guarantees are based on
the kernel approximation results of \cite{rahimi2008weighted} and
hence entail a relatively strong assumption on the Fourier
coefficients of the kernel predictors. In particular, they assume that
the Fourier transform of the optimal predictor decays at a faster rate
than the kernel density. We note that the standard assumption of
bounded Reproducing Kernel Hilbert Space (RKHS) norm does not imply
such a condition.

\cite{jain2013differentially} gave algorithms for DP predictions with
kernels. They considered scenarios where the goal is to privately
generate predictions (labels) on a small test set that admits no
privacy constraints. In such scenarios, their algorithms do not output
a classifier. The solution is interactive and limited to answering a
small number of classification queries, within some privacy
budget. \cite{jain2013differentially} also considered the
non-interactive case, however, their construction in this setting is
computationally inefficient and the resulting error guarantees are
dimension-dependent.

In this work, we give a new private construction
for shift-invariant kernel classifiers and derive an error bound that
nearly matches (up to a factor of $\wt O(1/\varepsilon)$) the optimal
non-private bound under standard problem setup (kernel predictors with
bounded RKHS norm). Our algorithm also runs in polynomial time. These
improvements over prior work hold thanks to the use of our private
algorithm for linear classifiers after approximating the kernel
together with a tighter error analysis that does not require any
non-standard assumptions. We also show how to achieve similar results
for polynomial kernels (which are not necessarily shift-invariant) by
using an embedding such as the JL-transform as well as other, more
efficient techniques, such as oblivious sketches.

\vspace{0.25cm}

We start with the introduction of some basic definitions
and concepts needed for our discussion.

\section{Preliminaries}
\label{sec:preliminaries}

We consider an input space $\cX$, a binary output space $\cY =
\curl{-1, +1}$ and a hypothesis set $\sH$ of functions mapping from
$\cX$ to $\Rset$.  We denote by $\sD$ a distribution over $\cZ = \cX
\times \cY$ and denote by $R_{\sD}(h)$ the generalization error and by
$\widehat R_S(h)$ the empirical error of a hypothesis $h \in \sH$:
\begin{align*}
R_{\sD}(h)  = \E_{z = (x, y) \sim \sD}[1_{yh(x) \leq 0}]
\qquad
\widehat R_S(h)  = \E_{z = (x, y) \sim S}[1_{y h(x) \leq 0}],
\end{align*}
\rnote{why not write $R$ as $R_\sD$. I think that would be more
  descriptive especially when we refer to different
  distributions/domains in the proofs. I adopted $R_{\sD}$ in section
  3.2} \thnote{I am fine with that. @mehryar, please let us know what
  you think.} where we write $z \sim S$ to indicate that $z$ is
randomly drawn from the empirical distribution defined by $S$. Given
$\rho \geq 0$, we similarly define the $\rho$-margin loss and
empirical $\rho$-margin loss of $h \in \sH$:
\begin{align*}
R^\rho_\sD(h) = \E_{z = (x, y) \sim \sD}[1_{yh(x) \leq \rho}] 
\qquad
\widehat R^\rho_S(h)  = \E_{z = (x, y) \sim S}[1_{yh(x) \leq \rho}] .
\end{align*}
The $\rho$-margin loss is not convex. Hence, we also consider
$\rho$-hinge loss to provide computationally-efficient algorithms. For
any $\rho > 0$, define $\rho$-hinge loss as
$\basichng^{\rho}(u)\triangleq \max\left(1-u/\rho,\, 0\right),~ u\in
\re$. Similar to the above definitions, given $\rho > 0$, for a sample
$S$, we define the $\rho$-hinge loss and empirical $\rho$-hinge loss as
\[
L^{\rho}_\sD(w) = 
\E_{z = (x, y) \sim \sD}[\basichng^{\rho}\left(y_i \,\langle w, x_i
  \rangle \right)]
\qquad
\emphng^{\rho}_S(w) = 
\E_{z = (x, y) \sim S}[\basichng^{\rho}\left(y_i \,\langle w, x_i
  \rangle \right)].
\]
\rnote{should we also define the counterparts of the above for
  hinge-loss here; they are currently defined at the beginning of sec
  3.2 (the efficient algorithm)?} \thnote{Moved. Should we try
  $\hat{L}^\rho_S$ to denote the empirical hinge loss? This will be
  consistent with empirical margin loss. }
In the context of learning, differential privacy is defined as follows.

\begin{defn}[Differential privacy]
  \label{def:user_dp}
  Let $\varepsilon, \delta \ge 0$. Let $\cA\colon (\cX \times \cY)^m
  \to \sH$ be a randomized algorithm. We say that $\cA$ is
  \emph{$(\varepsilon, \delta)$-DP} if for any measurable subset $O
  \subset \sH$ and all $S, S'\in(\cX \times \cY)^m$ that differ in one
  sample, the following inequality holds:
\begin{equation}
\P(\cA(S) \in O) \le e^\varepsilon \P(\cA(S') \in O) + \delta.
\end{equation}
If $\delta=0$, we refer to this guarantee as \emph{pure differential
privacy}.
\end{defn}

\section{Private Algorithms for Linear Classification with Margin Guarantees}
\label{sec:linear}

In this section, we present two private learning algorithms for linear
classification with margin guarantees: first, a computationally
inefficient pure DP algorithm, which we show benefits from relative
deviation margin bounds, next, a computationally efficient DP
algorithm with a dimension-independent bound expressed in terms of the
empirical $\rho$-hinge loss.

Let $\B^d(r)\triangleq \curl*{x\in \Rset^d \colon \norm{x}_2 \leq r}$
denote the Euclidean ball in $\Rset^d$ of radius $r$ and let $\cX
\subseteq \B^d(r)$ denote the feature space. We will use the shorthand
$\B^d$ for $\B^d(1)$. We consider the class of linear predictors over
$\cX$ defined by $\cH_\lin = \curl*{h_w \colon x \mapsto \tri{w, x}
  \mid w \in \B^d(\Lambda)}$. Note that one can represent the
general class of affine functions over $\Rset^d$ as linear functions
over $\re^{d+1}$ by simply mapping each $x\in \re^{d}$ to $\tilde{x} =
(x, 1) \in \Rset^{d + 1}$. Thus, without loss of generality, here we
will consider $\cH_\lin$. Here, we view $d$ as possibly much larger than the sample size $m$. We also note that even though the predictors in the input
class $\cH_\lin$ admit $\Lambda$-bounded norm, we do not constrain the
algorithm to output a predictor with bounded norm, which circumvents
the necessary dependence on the dimension in constrained DP
optimization \cite{BassilySmithThakurta2014}.

\subsection{Pure DP Algorithm for Linear Classification}

A standard method for designing differentially private algorithms for
a continuous hypothesis class is to apply the exponential mechanism
\cite{mcsherry2007mechanism} to a cover of the hypothesis
class. Since $\cH_\lin$ is $d$-dimensional, the size of a useful cover
is about $2^{\Omega(d)}$, thus, a direct application of the
exponential mechanism yields an $\Omega(d)$-bound; we give a simple
example illustrating that in Appendix~\ref{app:counter}.  Thus,
instead, we seek to reduce the size of the cover without impacting its
accuracy, using random projections. This results in a mapping $\Phi$
from $\Rset^d$ to a lower-dimensional space $\Rset^k$.

For linear classification, we wish to preserve intra-point distances
and angles, that is $x \cdot x' \approx \Phi x \cdot \Phi x'$ for
points $x$ and $x'$.  It is known that this property can be fulfilled
as a corollary of the Johnson-Lindenstrauss lemma
\cite[Theorem~109]{nelson2011sketching}. For completeness, we provide
a full proof in Appendix~\ref{app:useful}.  More interestingly, we show that we can
reduce the dimension to $\widetilde{O}(\Lambda^2 r^2/\rho^2)$, without
the error decreasing significantly. We then run the exponential
mechanism in this lower-dimensional space and next compute a
classifier $\tilde w$ in that space. We finally derive a classifier in
the original space by applying the transpose of the original
projection matrix $\Phi^T \tilde{w}$. Note that the final output
$\Phi^T \tilde{w}$ has expected norm
$\widetilde{O}\paren[\Big]{\frac{\sqrt{d}\rho}{r}}$ and may not lie in
$\B^d(\Lambda)$.

Algorithm~\ref{alg:prvmrg} gives the pseudocode of the full
algorithm. The algorithm and the analysis in this section include a
dimensionality reduction technique for mapping the feature vectors
from the input $d$-dimensional space to a $k$-dimensional space, where
$k = \widetilde{O}(\Lambda^2 r^2/\rho^2)$ for some $\rho \in (0, \Lambda
r]$. Hence, we will be dealing with ``compressed'' parameter vectors
  in $ \Rset^k$. To distinguish these two spaces, we will denote the
  empirical error and the empirical $\rho$-margin error in this
  $k$-dimensional space as $\widehat R^{(k)}_{S'}(w')$ and $\widehat
  R^{(k), \rho}_{S'}(w')$, respectively, where $w'\in \B^k$ and $S'\in
  \left(\Rset^k\times\cY\right)^m$.

\begin{restatable}{thm}{PrvMrgClass}
\label{thm:prv_mrg_class}
Algorithm~\ref{alg:prvmrg} is $\varepsilon$-differentially private.
For any $\beta \in (0, 1)$, with
probability at least $1 - \beta$ over the draw
of a sample $S$ of size $m$ from $\sD$, the solution $w^\prv$ it returns satisfies:
\begin{align*}
R_{\sD}(w^\prv)
& \leq \min\limits_{w\in \B^d(\Lambda)} \curl*{ \widehat R^\rho_S(w)+ O  \paren*{ \ \sqrt{\widehat R^\rho_S(w)  \paren*{ \ \frac{\Lambda^2 r^2\log^2\paren*{\frac{m}{\beta}}}{m\rho^2} + \frac{\log\paren*{\frac{1}{\beta}}}{m} }} + \Gamma}},\\
\text{where } \quad 
   \Gamma 
   & =  \frac{\Lambda^2 r^2\log\paren*{\frac{m}{\beta}}\log\left(\frac{\Lambda r}{\beta \rho}\right)}{\rho^2\varepsilon m}
   +  \frac{\Lambda^2 r^2\log^2\paren*{\frac{m}{\beta}}}{m\rho^2} + \frac{\log\paren*{\frac{1}{\beta}}}{m}.
\end{align*}
\end{restatable}
The proof is given in Appendix~\ref{app:prv_mrg_class_proof}. This result, although given for a computationally inefficient method, is
stronger than several previously known ones: First, it is an
$(\varepsilon ,0)$ pure differential privacy guarantee; second, it is
dimension-independent and furthermore, unlike prior work, the norm of
the optimal hypothesis does not appear in the bound.  Furthermore,
since it is a relative deviation margin bound, it smoothly
interpolates between the realizable case of $\widehat R^\rho_S(w) = 0$
and the case of $\widehat R^\rho_S(w) > 0$.  For a sample of size $m$,
the bound is based on an interpolation between a ${1}/{\sqrt{m}}$-term
that includes the square-root of the empirical margin loss as a factor
and another term in ${1}/{m}$. In particular, when the empirical
margin loss is zero, the bound only admits the ${1}/{m}$ fast rate
term. As a corollary, note that, up to constants, one can always
obtain privacy for $\varepsilon > 1$ essentially for free.

\begin{algorithm}[t]
	\caption{$\cA_{\prv\mrg}$: Private Learner of Linear Classifiers with Margin Guarantees} 
	\begin{algorithmic}[1]
		\REQUIRE Dataset $S=\left((x_1, y_1), \ldots, (x_m, y_m)\right) \in \left(\cX\times \{\pm 1\}\right)^m$; privacy parameter $\varepsilon >0$; margin parameter $\rho \in (0, \Lambda r]$; confidence parameter $\beta>0$.
        \STATE Let $k=O \left(\frac{\Lambda^2\,r^2\,\log\paren*{\frac{m}{\beta}}}{\rho^2} \right)$.
        \STATE Let $\Phi$ be a random $k\times d$ matrix with entries drawn i.i.d.\  uniformly from $\{\pm \frac{1}{\sqrt{k}}\}$.
        \STATE Let $S_\Phi=\left\{(x_\Phi, y) \colon (x, y)\in S\right\}$, where for any $x\in \Rset^d,$ $x_\Phi\triangleq \Phi x \in \Rset^k$. 
        \STATE Let $\cC$ be an $\frac{\rho}{10\,r}$-cover of $\B^k(2\Lambda)$. 
        \STATE Run the Exponential mechanism over $\cC$ with privacy parameter $\varepsilon$, sensitivity $1/m$, and score function $-\widehat R^{(k)}_{S_\Phi}(w)$ for $w\in\cC$, to select $\tilde{w} \in \cC$. \label{stp:exp-mech-learn} \thnote{Added the sensitivity.}
	    \RETURN $w^\prv= \Phi^{\top}\tilde{w}$, where $\Phi^{\top}$ denotes the transposition of $\Phi$. \label{stp:output}
	   	\end{algorithmic}
	\label{alg:prvmrg}
\end{algorithm}
\subsection{Efficient Private Algorithm for Linear Classification}
\ignore{
\subsection{Efficient Algorithm for Linear Classifiers}
}
\label{sec:effprvmargin}

\thnote{added few lines here, please feel free to edit.}  Since the
$\rho$-margin loss is not convex, minimizing it efficiently is
generally intractable. Instead, we devise a computationally efficient
algorithm, whose guarantees are expressed in terms of the empirical
$\rho$-hinge loss. Algorithm~\ref{alg:effprvmrg} shows the
pseudocode of our algorithm. We now discuss the key
steps of the algorithm.

\textbf{Fast JL-transform.} Our algorithm entails a
dimensionality reduction step (step~\ref{step:fast-JL}) as in
Algorithm~\ref{alg:prvmrg}. Here, we note that the new dimension $k$
is chosen to be $\wt O(m\varepsilon)$, which enables us to control the
influence of the dimensionality reduction on the empirical hinge
loss. We also note that this step is carried out via a fast
construction for the JL-transform (Lemma~\ref{lem:FJL}), which takes
$O(d\log(d))$ time, assuming $d> \varepsilon m$.

\textbf{Near linear-time DP convex ERM.}
After this step, we invoke an efficient algorithm for DP-ERM
(step~\ref{stp:dp-erm} in Algorithm~\ref{alg:effprvmrg}) to find an
approximate minimizer of the empirical $\rho$-hinge loss
$\emphng^{\rho}_{S_\Phi}(w)$, rather than using the exponential
mechanism to find an approximate minimizer for the empirical zero-one
loss $\widehat R_{S_\Phi}(w)$. To improve the running time of
step~\ref{stp:dp-erm}, we use the construction in
\cite[Algorithm~2]{bassily2021differentially} to solve DP-ERM in
near-linear time and with high-probability guarantee on the excess
empirical risk (see Algorithm~\ref{alg:nsgd} in
Appendix~\ref{app:algs}). The algorithm of
\cite{bassily2021differentially} is devised for DP-SCO with respect to
non-smooth generalized linear losses. It is based on a combination of
a smoothing technique via proximal steps 
and the phased SGD
algorithm \cite[Algorithm~2]{FeldmanKorenTalwar2020} for smooth
DP-SCO. The algorithm of \cite{bassily2021differentially} can be used
for DP-ERM if it is fed with a sample from the empirical distribution
of the dataset. However, the privacy guarantee requires a careful
privacy analysis that takes into account the fact that such a sample
may contain duplicate entries from the original dataset.

Moreover, since the original algorithms in
\cite{FeldmanKorenTalwar2020,bassily2021differentially} provide only
expectation guarantees and we aim at high-probability learning bounds,
we use a standard private confidence-boosting technique to provide
high-probability guarantee on the excess risk of our variant. We
summarize the guarantees of this variant in the following lemma. The
details of the construction and the proof of the lemma below can be
found in Appendix~\ref{app:algs}.

\begin{restatable}{lem}{effprverm}
  \label{lem:eff-prv-erm}
Let $m\in \mathbb{N}$, $0<\delta<\frac{1}{m}$, and $0<\varepsilon\leq \log(1/\delta)$. Algorithm~\ref{alg:nsgd}
(Appendix~\ref{app:algs}) is $(\varepsilon, \delta)$-DP. Let $\beta
\in (0, 1)$. Let $k\in \mathbb{N}$, and $\tilde r, \Lambda >0$. Let
$\wt S \in \paren*{\B^k(\tilde r)\times \{\pm 1\}}^m$ be the input
dataset and $\B^k(\Lambda)$ be the parameter space. With probability
$1-\beta$ over the randomness in Algorithm~\ref{alg:nsgd}, the output
$\tilde w$ satisfies
\begin{align*}
    \emphng^{\rho}_{\wt S}(\tilde w)&\leq \min\limits_{w\in \B^k(
        \Lambda)}\emphng^{\rho}_{ \wt S}(w) +\frac{\Lambda \tilde r}{\rho}\cdot O\paren*{\frac{1}{\sqrt{m}}+\frac{\sqrt{k}\log^{\frac{3}{2}}(\frac{1}{\delta})\log(\frac{1}{\beta})}{\varepsilon m}}.
\end{align*}
Moreover, Algorithm~\ref{alg:effprvmrg} requires $O(m\log(m)\log(1/\beta))$ gradient computations. 
\end{restatable}


\begin{algorithm}[th]
	\caption{$\cA_{\eff\prv\mrg}$: Efficient Private Learner of Linear Classifiers with Margin Guarantees} 
	\begin{algorithmic}[1]
		\REQUIRE Dataset $S=\left((x_1, y_1), \ldots, (x_m, y_m)\right) \in \left(\B^d(R)\times \{\pm 1\}\right)^m$; privacy parameters $\varepsilon, \delta$; norm bound $\Lambda$;  margin parameter $\rho \in (0, \Lambda r]$; confidence parameter $\beta>0$.
        \STATE Let $k=\frac{\varepsilon m \log(m/\beta)}{\log^{\frac{3}{2}}(1/\delta)\log(1/\beta)}$. \label{step:set_k}
        \STATE Let $\Phi$ be a random $k\times d$ matrix from the construction in Lemma~\ref{lem:FJL}.\label{step:FJL1}
        \thnote{Replaced with the fast FL transform}
        \STATE Let $S_\Phi=\left\{\big(\Pi_{\B^k(2r)}(x_\Phi), y\big)\colon (x, y)\in S\right\}$, where for any $x\in \Rset^d,$ $x_\Phi\triangleq \Phi x \in \Rset^k$ and $\Pi_{\B^k(2r)}$ is the Euclidean projection on $\B^k(2r)$.  \label{step:fast-JL}
        \STATE Privately solve the convex ERM problem: $\argmin\limits_{w\in \B^k(
        2\Lambda)}\emphng^{\rho}_{S_\Phi}(w)$ via Algorithm~\ref{alg:nsgd} (Appendix~\ref{app:algs}) 
        and return  $\tilde{w}\in \B^k(2\Lambda)$. \label{stp:dp-erm} 
	    \RETURN $w^\prv= \Phi^{\top}\tilde{w}$, where $\Phi^{\top}$ denotes the transposition of $\Phi$. \label{step:output} \label{stp:final_output}
	   	\end{algorithmic}
	\label{alg:effprvmrg}
\end{algorithm}

We now state our main result in this section, which we prove in Appendix~\ref{app:prv_mrg_hinge_proof}.

\begin{restatable}{thm}{PrvMrgHinge}
\label{thm:prv_mrg_hinge}
Let $0<\delta<\frac{1}{m}$ and $0<\varepsilon\leq \log(1/\delta)$. Algorithm~\ref{alg:effprvmrg} is
$(\varepsilon, \delta)$-DP. Let $\beta \in (0, 1)$. Let $S\sim \sD^m$
for a distribution $\sD$ over $\B^d(r)\times \{\pm
1\}$. Algorithm~\ref{alg:effprvmrg} outputs $w^\prv\in \Rset^d$ such
that with probability at least $1-\beta,$ we have
\begin{align*}
    R_\sD(w^\prv)&\leq \min\limits_{w\in \B^d(\Lambda)}\emphng^{\rho}_S(w)
    + 
    O\paren*{\sqrt{\frac{\log(1/\beta)}{m}}+\frac{\Lambda r}{\rho}\left(\frac{1}{\sqrt{m}}+\frac{\sqrt{\log(\frac{m}{\beta})\log(\frac{1}{\beta})}\log^{\frac{3}{4}}(\frac{1}{\delta})}{\sqrt{\varepsilon m}}\right)}.
\end{align*}
Moreover, Algorithm~\ref{alg:effprvmrg} runs in time $O\paren*{m
  d\log\paren*{\max(d, m)}+\varepsilon
  m^2\log(m)/\log^{\frac{3}{2}}(1/\delta)}$.  \thnote{I believe there
  is an $\varepsilon m d $ term due to step $5$ of the algorithm.}
\thnote{I think there are a few terms missing here e.g.,
  $\log(m/\beta)$ and the terms that are independent of $\varepsilon$,
  which matters when $\varepsilon > 1$. Not changing as we might this
  theorem with the high probability theorem from Raef.} \rnote{I fixed
  it. Also, the bound reflect the setting of $k$ for the fast JL. We
  need to make explicit reference to that in
  Algorithm~\ref{alg:effprvmrg} or the discussion before the
  algorithm.}
\end{restatable}

\section{Private Algorithms for
Kernel-Based Classification with Margin Guarantees}
\label{sec:kernels}

In this section, we present private algorithms with
margin guarantees for kernel-based predictors \cite{ScholkopfSmola2002, ShaweTaylorCristianini2004}.
We first consider a continuous, positive definite, shift-invariant kernel
$K\colon \cX \times \cX \to \re$, where $K(x, x) = r^2$ for all
$x \in \cX$. The associated feature map is defined as $\psi(x) \triangleq
K(\cdot, x),~ x\in \cX$, where $\cX \subseteq \B^d(r)$. 
\paragraph{Overview of the technique.}
Our approach is based on approximating the feature map $\psi$ by a
finite-dimensional feature map $\hpsi: \cX \to \B^{2D}\,(r)$
determined via Random Fourier Features (RFFs). The dimension $2D$ of
the approximate feature map is chosen to be sufficiently large to
ensure that for all pairs of feature vectors $x_i, x_j$ in a training
set $S=\left((x_1, y_1), \ldots, (x_m, y_m)\right)$, we have $\lvert
\langle \hpsi(x_i), \hpsi(x_j)\rangle - K(x_i, x_j)\rvert\lesssim
\frac{1}{m}$ with high probability over the randomness of $\hpsi$ (due
to RFFs). This suffices to derive an upper bound (margin bound) on the
true error of a finite-dimensional linear predictor trained on the
sample made of the labeled points
$\paren*{\hpsi(x), y)}$, $(x, y) \in S$, that is
essentially the same as the margin bound known for the kernel
classifier. Hence, in effect, we reduce the problem to that of
learning a linear classifier in a $2D$-dimensional space, which we
can solve privately using Algorithm~\ref{alg:effprvmrg}. Note that the
output predictor in this case is a finite-dimensional linear function
rather than a function in the Reproducing Kernel Hilbert Space. A full
description of our DP learner of kernel classifiers is given in
Algorithm~\ref{alg:prvkermrg} below.

\paragraph{Bochner's Theorem and RFFs.}
Since the kernel $K$ is shift-invariant, it can be expressed as $K(x,
x') = r^2 \Bar{K}(x - x'), x, x' \in \cX$ for some function $\Bar{K}
\colon \cZ \to \re$, where $\cZ=\curl*{z = x - x' \colon x, x' \in
  \cX}$. Moreover, since $K$ is positive-definite, $\Bar{K}$ is the
Fourier transform of a probability distribution $P_{\Bar{K}}$:
\[
\Bar{K}(x)=\int_{\cX}e^{i\langle \omega, x\rangle}P_{\Bar{K}}(\omega) d\omega. 
\]
This follows from Bochner's Theorem
\cite{rudin2017fourier}. Random Fourier Features (RFFs) provide a simple
method introduced in \cite{RahimiRecht2007} to approximate
kernel feature maps in a data-independent fashion. The idea is based
on Bochner's theorem. In particular, we first sample $\omega_1,
\ldots, \omega_D$ independently from the probability distribution
$P_{\Bar{K}}$. Then, we define an approximate feature map as follows:
\begin{align}
 \hpsi(x)&\triangleq \frac{r}{\sqrt{D}}\left(\cos{\langle \omega_1, x\rangle}, \sin{\langle \omega_1, x\rangle}, \ldots \cos{\langle \omega_D, x\rangle}, \sin{\langle \omega_D, x\rangle}\right), \quad~ \forall x\in \cX.\label{eq:approx-feature-map}
\end{align}
For $D$ sufficiently large, it can be shown that $\langle \hpsi(x),
\hpsi(x') \rangle$ concentrates around $K(x, x')$ for all pairs $x, x'
\in \cX$ \cite[Theorem~6.28]{mohri2018foundations}. In our analysis
below, we only need that concentration to hold uniformly over pairs
$x, x'$ from a fixed training set rather than uniformly over all pairs
$x, x' \in \cX$. This leads to a simpler approximation guarantee,
which we formally state below.

\begin{thm}\label{thm:RFF-approx}
Let $S_{\cX}=\left(x_1, \ldots, x_m\right) \in \cX^m$. Let $K$ be a
shift-invariant, positive definite kernel, where $K(x, x)=r^2,~
\forall x\in \cX$. Let $P_{\Bar{K}}$ be the probability distribution
associated with $K$. Suppose $\omega_1, \ldots, \omega_D$ are drawn
independently from $P_{\Bar{K}}$. With probability $1$, we have
$\norm{\hpsi(x)}_2=r,~ \forall x\in \cX$. For any $\beta \in (0, 1)$,
with probability at least $1-\beta$, for all $i, j \in [m]$ such that
$i\neq j$ we have
$$\lvert \langle \hpsi(x_i), \hpsi(x_j)\rangle - K(x_i, x_j)\rvert\leq 2r^2\sqrt{\frac{\log\paren*{\frac{m}{\beta}}}{D}}.$$
\end{thm}

\begin{proof}
The first assertion related to $\norm{\hpsi(x)}_2~ \forall x\in \cX$
follows directly from the definition of $\hpsi(x)$ and a basic
trigonometric identity. The proof of the second assertion about
the inner products follows from the identity
$\ex{\omega_1,\ldots, \omega_D}{\langle \hpsi(x_i),
  \hpsi(x_j)\rangle}=K(x_i, x_j)$ that
  holds for all $x_i, x_j$, the application of  Hoeffding's
bound combined with the union bound over all $\approx m^2$ pairs $x_i,
x_j \in S_{\cX}$. The unbiasedness of $\langle \hpsi(x_i),
\hpsi(x_j)\rangle$ follows from the fact that the expectation is the
Fourier transform of $r^2 P_{\Bar{K}}(\omega)$, which, by Bochner's
Theorem, is $r^2 \Bar{K}(x_i - x_j)=K(x_i, x_j)$.
\end{proof}



\begin{algorithm}[t]
	\caption{$\cA_{\prv\krn\mrg}$: Efficient Private Learner of Kernel Classifiers with Margin Guarantees}
	\begin{algorithmic}[1]
		\REQUIRE Dataset $S=\left((x_1, y_1), \ldots, (x_m, y_m)\right) \in \left(\cX\times \{\pm 1\}\right)^m$; shift-invariant, positive definite kernel $K:\cX\times\cX \rightarrow \re$ with $K(x, x)=r^2$ for all $x\in \cX$, privacy parameters $\varepsilon, \delta$; Reproducing kernel Hilbert space (RKHS) norm bound $\Lambda$; margin parameter $\rho \in (0, 2\Lambda r]$; confidence parameter $\beta>0$.
        \STATE Let $P_{\Bar{K}}$ be the probability distribution associated with $K$. 
        \STATE Let $D= m^2\log(2m/\beta)$. \label{step:set_D}
        \STATE Draw $\omega_1, \ldots, \omega_D$ independently from $P_{\Bar{K}}$.
        \STATE Let $S_{\hpsi}=\left((\hpsi(x_i), y_i) \colon i \in [m]\right)$, where $\hpsi$ is as defined in (\ref{eq:approx-feature-map}).\label{step:S_hpsi} 
	    \STATE $w^\prv \leftarrow \cA_{\eff\prv\mrg} \left(S_{\hpsi}\,,\, \varepsilon, \delta, 2\Lambda,\, \rho, \beta/2\right)$, where $\cA_{\eff\prv\mrg}$ is the private learner described in Algorithm~\ref{alg:effprvmrg}. Note that the input dimension to $\cA_{\eff\prv\mrg}$ is $2D$ rather than $d$, and the norm bound parameter is $2\Lambda$. 
	    \RETURN Private predictor $h_{w^\prv}:\cX \rightarrow \re$ defined as $\forall x\in \cX, ~ h^{\hpsi}_{w^\prv}(x)\triangleq \langle w^\prv, \hpsi(x)\rangle.$
	   	\end{algorithmic}
	\label{alg:prvkermrg}
\end{algorithm}

\thnote{Since we are not defining some of the quantities in the
  theorem e.g., RKHS etc, might be good to add a reference. Any
  suggestions?} \mnote{If it is not already the case, at least for the
  first occurrence, we should write Reproducing Kernel Hilbert
  Space. As for references, I suggest the kernel methods book of Smola
  and Scholkopf and that of Shawe-Taylor. You could also reference the
  kernel section of our textbook, which I think covers all of
  these.}\rnote{we write Reproducing Kernel Hilbert Space explicitly
  and abbreviate it as (RKHS) in the intro.}

We now state our main result, which we prove in Appendix~\ref{app:kernels}.
\begin{restatable}{thm}{PrivKernel}\label{thm:PrivKernel}
Let $r> 0$. Let $K:\cX\times\cX \rightarrow \re$ be a shift-invariant,
positive definite kernel, where $K(x, x)=r^2$ for all $x\in \cX$. For
any $\varepsilon> 0$ and $\delta \in (0, 1)$,
Algorithm~\ref{alg:prvkermrg} is $(\varepsilon,
\delta)$-differentially private.
Define $\cH_\Lambda\triangleq \{h\in \mathbb{H}:
\norm{w}_{\mathbb{H}}\leq \Lambda\}$, where
$\norm{\cdot}_{\mathbb{H}}$ is the norm corresponding to the
reproducing kernel Hilbert space (RKHS)
$\mathbb{H}$ associated with the kernel $K$. Let $\beta \in (0, 1)$. Given an input sample $S$ of $m$ examples drawn i.i.d.\ from a distribution $\cD$ over $\cX\times \{\pm 1\}$, Algorithm~\ref{alg:prvkermrg} outputs $h^{\hpsi}_{w^\prv}$ such that with probability at least $1-\beta,$ we have
\begin{align*}
    R_\sD(h^{\hpsi}_{w^\prv})&\leq \min\limits_{h\in \cH_\Lambda}\emphng^{\rho}_S(h) +   O\paren*{\sqrt{\frac{\log(1/\beta)}{m}}+\frac{\Lambda r}{\rho}\left(\frac{1}{\sqrt{m}}+\frac{\sqrt{\log(\frac{m}{\beta})\log(\frac{1}{\beta})}\log^{\frac{3}{4}}(\frac{1}{\delta})}{\sqrt{\varepsilon m}}\right)},
\end{align*}
where, for any $h\in \cH_\Lambda$, $\emphng^{\rho}_S(h)\triangleq \frac{1}{m} \sum_{i=1}^m \basichng^{\rho}\left(y_i \,\langle h, \psi(x_i) \rangle_{\,\mathbb{H}} \right)$, where $\psi$ is the feature map associated with the kernel $K$ and $\langle \cdot, \cdot \rangle_{\,\mathbb{H}}$ is the inner product associated with the RKHS $\,\mathbb{H}$.
\end{restatable}
\rnote{look at the results of the paper on polynomial kernels
  \cite{PhamPagh2013}; we can add remark} \mnote{I am adding a
  discussion here, but feel free to change or move it.}  \thnote{I
  think it will be helpful to mention explicitly that
  $(h^{\hpsi}_{w^\prv}) \notin \mathbb{H}$ and add some
  discussion.}\rnote{I don't think it's a big issue especially that
  prior works had similar approach, but I added a line on that towards
  the end of the ``overview of the technique'' paragraph.}

\noindent \textbf{Polynomial kernels:} Our results can be
extended to polynomial kernels using a different approach to construct
a finite-dimensional approximation of the kernel. A polynomial kernel
of degree $p$, denoted as $\kappa_p$, can be expressed as $\kappa_p(x,
x')= \paren*{\langle x, x'\rangle + c}^p,~ x, x' \in \cX$ and $c>0$ is
some constant. Note that a feature map $\psi_p$ associated with
such a kernel can be expressed as a vector in $\Rset^{\Bar{d}},$ where
$\Bar{d}=O(d^p)$. In particular, $\psi_p(x)$ is the vector of all
monomials of a $p$-th degree polynomial. Ignoring computational
efficiency considerations (or when $p$ is a small constant), there is
a simpler private construction than the one used for shift-invariant
kernels. In that case, we can directly use the JL-transform to embed
$\{\psi_p(x_1), \ldots, \psi_p(x_m)\}$ into a $k$-dimensional subspace
exactly as in Section~\ref{sec:effprvmargin}, which would
result in a $k$-dimensional approximation of the kernel (by the properties
of the JL-transform). Hence, we can directly use
Algorithm~\ref{alg:effprvmrg} on the dataset $S_{\psi_p}\triangleq
\paren*{(\psi_p(x_1), y_1), \ldots, \psi_p(x_m), y_m)}$. We therefore
obtain the same bound on the expected error as above except that $r$
would then be $r^p$. That dependence on $r^p$ is inherent
in this case even non-privately since $\kappa_p(x, x)$ can be as large
as $r^p$. 
However, as discussed below, more efficient
solutions can designed for approximating these
and many other kernels.

\textbf{Further extensions.}
Our work can directly benefit from the method of
\cite{LeSarlosSmola13}, which is computationally faster than that of
\cite{RahimiRecht2007}, $O((m + d) \log d)$, instead of $O(md)$.
Their technique also extends to any kernel that is a function of an
inner product in the input space.\ignore{  An approach seeking to improve the
time complexity of kernel approximation in the context of Nystr\"{o}m
approximations, has also been presented by \cite{MuscoMusco2017}.
}
\ignore{
For polynomial kernel approximations, sketching techniques based on a
tensor sketch were first presented by \cite{PhamPagh2013}.  Note that
polynomial kernel approximations are important since many kernels can
be expressed as a sum of polynomial kernels via a Taylor expansion and
thereby also approximated via such methods. This technique was later
extended by \cite{AvronNguyenWoodruff2014} to several
other kernel approximation problems.  The running-time complexity of
these algorithms, however, depends exponentially on the degree of the
polynomial.}
We can further use, instead of the JL-transform, the \emph{oblivious
sketching} technique of
\cite{AhleKapralovKnudsenPaghVelingkerWoodruffZandieh2020} from
numerical linear algebra, which builds on on previous work by
\cite{PhamPagh2013} and \cite{AvronNguyenWoodruff2014}, to design
sketches for polynomial kernel with a target dimension that is only
polynomially dependent on the degree of the kernel function, as well
as a sketch for the Gaussian kernel on bounded datasets that does not
suffer from an exponential dependence on the dimensionality of input
data points.
More recently, \cite{SongWoodruffYuZhang2021} presented new oblivious
sketches that further considerably improved upon the running-time
complexity of these techniques. Their method also applies to other
\emph{slow-growing} kernels such as the neural tangent (NTK) and
arc-cosine kernels.

\section{Private Algorithms for Learning Neural Networks with Margin Guarantees}
\label{sec:NNs}

\setlength{\intextsep}{-4pt}
\setlength{\columnsep}{6pt}%
\begin{wrapfigure}{R}{1.2in}
  \centering\includegraphics[scale=.1]{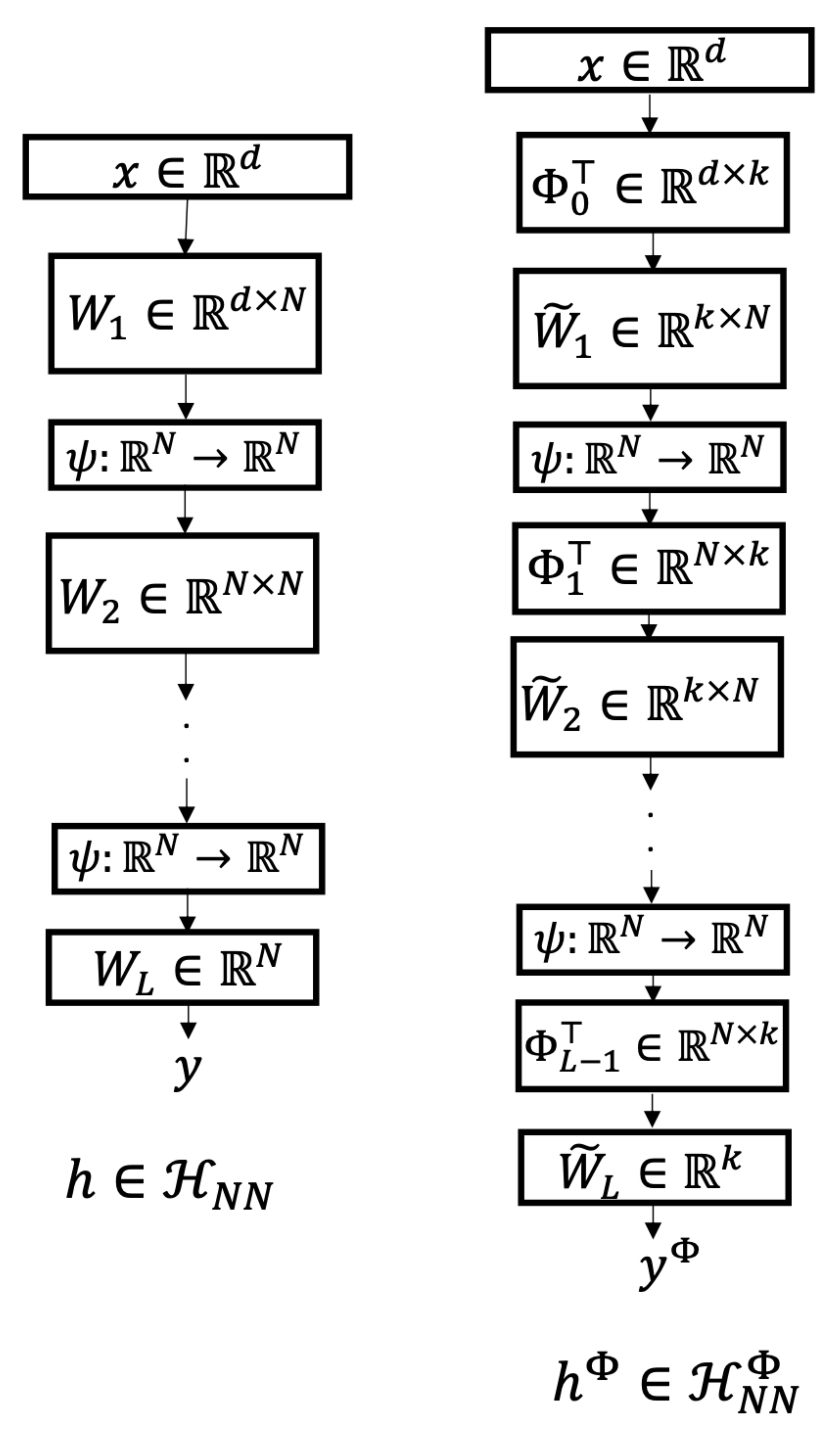}
  \vskip -.1in
  \caption{Illustration of the neural networks before and after JL-transforms.}
  \vskip -.25in
  \label{fig:NN}
\end{wrapfigure}

In this section, we describe a private learning algorithm that
benefits from favorable margin guarantees when run with a family of
neural networks with a very large input dimension.

We consider a family $\sH_{\sf NN}$ of $L$-layer feed-forward neural
networks defined over $\B^d(r)$, with $d$ potentially very large
compared to the sample size $m$.
A function $h$ in $\sH_{\sf NN}$ can be viewed as a
cascade of linear maps composed with a non-linear activation
function (see Figure~\ref{fig:NN}, left column).
Here, $W_1, \ldots, W_L$ are the weight matrices defining the network and $\psi$ is a non-linear activation. For simplicity, the width (number of neurons) in each hidden layer, denoted by $N$, is assumed to be the same for all the layers. Also, we assume that the output of the network is a real scalar and hence we have $W_L\in \Rset^N$. Furthermore, we assume no activation in the output layer. We also assume the same activation $\psi\colon \Rset^N\rightarrow \Rset^N$
for all layers and choose it to be a sigmoid function: 
for any $u = (u_1, \ldots, u_N)\in
\Rset^N,$ $\psi(u) = (\sigma_\eta(u_1), \ldots, \sigma_\eta(u_N)),$ for
some $\eta >0$, where $\sigma_\eta(a) = \frac{1 - e^{-\frac{\eta
      a}{2}}}{1 + e^{-\frac{\eta a}{2}}}$, $a \in \Rset$. Note that
$\sigma_\eta$ is $\eta$-Lipschitz and thus $\psi$ is $\eta$-Lipschitz
with respect to $\norm{\cdot}_2$: for any $u, v\in \Rset^N,
\norm{\psi(u)-\psi(v)}_2\leq \eta \norm{u-v}_2$. A typical choice for
$\eta$ in practice is $\eta=1$, but we will keep the dependence on
$\eta$ for generality.
We define $\sH_{{\sf NN}^{\Lambda}}$ as the subset of $\sH_{\sf NN}$ with weight matrices that are $\Lambda$-bounded in their Frobenius norm: for all $j \in [L]$, $\norm{W_j}_F \leq \Lambda$ for some $\Lambda >0$.  


We design a pure DP algorithm for learning
$L$-layer feed-forward networks in $\sH_{{\sf NN}^{\Lambda}}$ that
benefits from the following margin-based guarantee.

\begin{restatable}{thm}{MrgNNs}
\label{thm:MrgNNs}
Let $\varepsilon>0, \beta \in (0, 1),$ and $\rho >
0$. Then, there is an $\varepsilon$-DP algorithm 
which returns an $L$-layer network $h^\prv$ with $N$
neurons per layer that 
with probability at least $1 - \beta$
over the draw of a sample $S\sim \sD^m$ and the internal randomness of the algorithm admits the following guarantee:
\begin{align*}
  R_{\sD}(h^\prv)
  &\leq \min\limits_{~~~ h\in \sH_{{\sf NN}^{\Lambda}}}\h R_S^{\rho}(h)
  + O\paren*{\frac{r (2\eta \Lambda)^L\,\sqrt{N \theta}}{\rho \sqrt{m}}
    +\frac{r^2 (2\eta\Lambda)^{2L} \, N\theta}{\rho^2 \varepsilon m}},
\end{align*}
where $\theta = \log(Lm/\beta)\log(r(\eta\Lambda)^L/\rho)$\ignore{, $\h
R_S^{\rho}$ is the empirical $\rho$-margin loss, and $\sH_{{\sf
    NN}^{\Lambda}}$ is the family of $L$-layer networks described
above (with $N$ neurons per layer, sigmoid activation with parameter
$\eta$, and weight matrices $W_1, \ldots, W_L$ with Frobenius norm
bound $\Lambda$ on each of them)}.
\end{restatable}
Note that this  guarantee is independent of $d$ and, assuming $L$ is a constant, the bound scales roughly as 
   $\sqrt{\frac{N }{\rho^2 m}}+\frac{N
    }{\rho^2\varepsilon m}$, where $\rho$ is the confidence-margin
parameter and $\varepsilon$ is the privacy parameter. Note that, for
$\varepsilon \approx 1$, the bound scales with $\sqrt{\#\text{
    neurons}}$, which is more favorable than standard bounds obtained
via a uniform convergence argument, which depend
on $d$, as well as the total number of edges $\Omega(N^2)$, in addition to a
similar dependence on $\Lambda^L$.

\textbf{Our construction.}  Our DP learner is based on using $L$ embeddings
$\Phi_0 \in \Rset^{k\times d}, \ldots,$ \mbox{$\Phi_{L-1} \in \Rset^{k\times
  N}$} given by data-independent JL-transform matrices to reduce the
dimension of the inputs in each layer, including the input layer, to
$k = O \paren*{\frac{r^2 (2 \eta \Lambda)^{2L}}{\rho^2}}$. We
randomly generate a set $\Phi = (\Phi_0, \ldots, \Phi_{L-1})$ of $L$
independent JL matrices whose dimensions are described as above. 
We let $\sH_{\sf NN}^{\Phi}$ denote the family of
$L$-layer neural networks, where each network $h^{\Phi}\in
\sH^{\Phi}_{\sf NN}$ is associated with weight matrices $\Phi^\top_0\,
\wt W_1, \ldots, \Phi^\top_{L-1}\,\wt W_L$ for $\wt W_j \in
\Rset^{k\times N}, j\in [L-1]$, and $\wt W_L \in \Rset^{k\times 1}$ (see Figure~\ref{fig:NN}, right column). 
We define $\sH^{\Phi}_{{\sf NN}^{2\Lambda}}\subset \sH^{\Phi}_{\sf
  NN}$ where $\norm{\wt W_j}_F\leq 2\Lambda$ for all $j\in [L]$. We
start by creating a $\gamma$-cover $\cC$ of the product space of the
matrices $\wt W_1, \ldots, \wt W_L$ associated with
$\sH^{\Phi}_{{\sf NN}^{2\Lambda}}$, where $\gamma = 
O\paren*{\frac{\rho}{r(4\eta\Lambda)^{L}}}$.
$\cC$ is a $\gamma$-cover of $\B^{k\times N}(2\Lambda)\times \ldots \times 
\B^{k\times N}(2\Lambda)\times\B^k(2\Lambda)$ with respect to
$\sqrt{\sum_{j=1}^L\norm{\cdot}^2_F}$. We define $\h\sH^{\Phi}_{{\sf
    NN}^{2\Lambda}}\subset \sH^{\Phi}_{{\sf NN}^{2\Lambda}}$ to be the
corresponding family of networks whose associated matrices are in
$\cC$. Given an input dataset $S = \paren*{(x_1, y_1), \ldots, (x_m,
  y_m)} \in \paren*{\B^d(r)\times \{\pm 1\}}^m$, we then run the
exponential mechanism over $S$ with privacy parameter $\varepsilon$,
the score function being the empirical zero-one loss $\h R_S(h)\colon h \in
\h\sH^{\Phi}_{{\sf NN}^{2\Lambda}}$, and the sensitivity $1/m$, to return
a neural network $h^\prv\in \h\sH^{\Phi}_{{\sf NN}^{2\Lambda}}$.

\section{Label Privacy Algorithms with Margin Guarantees}
\label{sec:label}

In many tasks, the features are public information
and only the labels are sensitive and need to be protected. Several recent
publications have suggested to train learning models with
differential privacy for labels for these tasks, while treating
features as public information \cite{ghazi2021deep,
  esfandiari2021label}. This
motivates the following definition of label differential privacy.

\begin{defn}[Label differential privacy]
  \label{def:user_ldp} Let $\varepsilon, \delta 
\ge 0$. Let $\cA\colon (\cX \times \cY)^m \to \sH$ be a (potentially
randomized) mechanism. We say that $\cA$ is \emph{$(\varepsilon,
\delta)$-label-DP}
  if for any measurable subset $O \subset \sH$ and all $S, S'\in(\cX
  \times \cY)^m$ that differ in one label of one sample, the following
  inequality holds:
\begin{equation}
\P(\cA(S) \in O) \le e^\varepsilon \P(\cA(S') \in O) + \delta.
\end{equation}
\end{defn}
\rnote{Can we say in words and briefly what we mean by label DP rather
  than writing the whole definition? It would save us some space, and
  the notion is quite obvious.}

\cite{ghazi2021deep} gave an algorithm for deep learning with
label differential privacy in the local differential privacy
model. \cite{yuan2021practical} proposed and evaluated algorithms for
label differential privacy in conjunction with secure multiparty
computation. \cite{esfandiari2021label} presented a clustering-based
algorithm for label differential privacy. There are several other
works which show pitfalls on label differential privacy
\cite{busa2021pitfalls, busa2021population}.

Here, we design a simple algorithm for label differential privacy, which
we show benefits from margin guarantees for any hypotheses class
with finite fat-shattering dimension, including the class of
linear classifiers, neural networks, and ensembles
\cite{bartlett1999generalization}.

We first introduce some definitions needed to describe our
algorithm. Fix $\rho > 0$. Define the $\rho$-truncation function
$\beta_\rho \colon \Rset \to [-\rho, +\rho]$ by $\beta_\rho(u) =
\max\curl*{u, -\rho} 1_{u \leq 0} + \min\curl*{u, +\rho} 1_{u \geq
  0}$, for all $u \in \Rset$. For any $h \in \sH$, we denote by
$h_\rho$ the $\rho$-truncation of $h$, $h_\rho = \beta_\rho(h)$, and
define $\sH_\rho = \curl*{h_\rho \colon h \in \sH}$. For any family of
functions $\sF$, we also denote by $\cN_\infty(\sF, \varepsilon,
x_1^m)$ the empirical covering number of $\sF$ over the sample $(x_1,
\ldots, x_m)$ and by $\sC(\sF, \varepsilon, x_1^m)$ a minimum
empirical cover. With these definitions, the algorithm is given in
Algorithm~\ref{alg:marglabel}. The algorithm uses an exponential
mechanism over a cover of truncated hypotheses sets.

\begin{algorithm}[t]
	\caption{$\cA_\labmarg$: Private learning algorithm under label-privacy}
	\begin{algorithmic}[1]
		\REQUIRE Dataset $S=\left((x_1, y_1), \ldots, (x_m, y_m)\right) \in \left(B^d \times \{\pm 1\}\right)^m$; privacy parameter $\varepsilon >0$; margin parameter $\rho$.
		\STATE Compute the $\rho/2$ minimal cover $\hH = \sC(\sH_\rho, \rho/2, x_1^m)$.
        \STATE Run the Exponential mechanism with privacy parameter $\varepsilon$, sensitivity $1/m$, and score function $-\widehat R^{\rho/2}_S(h),~h\in\hH$ to select $\hprv \in \hH$.
	    \RETURN $\hprv$.
	   	\end{algorithmic}
	\label{alg:marglabel}
\end{algorithm}

\begin{restatable}{thm}{MargLabel}
\label{thm:label_dp}
Algorithm~\ref{alg:marglabel} is $\varepsilon$-label-DP. Let $\sD$ be
a distribution on $\cX\times \cY$ and suppose $S\sim \sD^m$.  Let
$c=17$ and $d = \fat_{\frac{\rho}{32}}(\sH)$ and $\fatshatteringbound
= 1 + d \log_2(2c^2m) \log_2\frac{2cem}{d} + \log
\frac{2}{\beta}$. For any $\beta \in (0, 1)$,
with probability at least $1-\beta,$ the output $w^\prv$ satisfies:
\[
R_\sD(\hprv)\leq \min_{h \in \sH} \paren*{ \widehat R^\rho_S(h)
+  2\sqrt{\min_{h \in \sH}\widehat R^\rho_S(h)}
\sqrt{ \frac{\fatshatteringbound}{m}} }
+ \frac{ 2\fatshatteringbound}{m}
+ \frac{64\fatshatteringbound \log\paren*{\frac{2}{\beta}}}{\varepsilon m}.
\]
\end{restatable}

While our algorithm is computationally inefficient, it admits strong
theoretical guarantees. First, it is an $(\varepsilon ,0)$ pure
label-differential privacy guarantee. Second, it is
dimension-independent. Furthermore, the relative deviation margin
bound it benefits from smoothly interpolates between the realizable
case of $R^\rho_S(w) = 0$ and the case of $R^\rho_S(w) > 0$. As a
corollary, note that up to constants one can always get privacy for
$\varepsilon > 1$ for free. Finally, observe that this bound holds not
only for linear classes, but also for any hypothesis set with favorable
$\rho$-fat-shattering dimension. In particular, we can use known upper bounds for the $\rho$-fat-shattering
dimension of feed-forward neural networks \cite{bartlett1999generalization}
to derive label-privacy guarantees for training neural networks.

\section{Conclusion}

We presented a series of new differentially private algorithms
with dimension-independent margin guarantees, including algorithms for
linear classification, kernel-based classification, or learning with a
family of feed-forward neural networks, and label DP learning with
general hypothesis sets.
Our kernel-based algorithms can be extended to non-linear
classification with many other kernels, including a variety of kernels
that can be approximated using polynomial kernels, using techniques
based on oblivious sketching. Our study of DP algorithms with margin
guarantees for a family of neural networks can be viewed as an
initiatory step that could serve as the basis for a more extensive
analysis of DP algorithms for broader families of neural networks.

\ignore{
Such techniques and ideas can be further used to devise private
learning algorithms for neural networks, by viewing activation
functions as feature mappings and post-activation inner products
as kernels that can be approximated by polynomial kernels using
a Taylor expansion of the activation functions. We have already
initiated this study, which we hope to present in the near future.
More generally, the algorithmic and analytical techniques used
in this work could be leveraged in the design of other new
DP learning algorithms with strong margin guarantees.
}

\section*{Acknowledgements}\label{sec:ack}
This work is done while RB was a visiting scientist at Google, NY. 
RB's research at OSU is supported by NSF Award AF-1908281, NSF Award 2112471, Google Faculty Research Award, and the OSU faculty start-up support.
\newpage

\bibliographystyle{alpha} 
\bibliography{dpm}

\newpage
\appendix

\renewcommand{\contentsname}{Contents of Appendix}
\tableofcontents
\addtocontents{toc}{\protect\setcounter{tocdepth}{3}} 
\clearpage

\section{Useful Lemmas}
\label{app:useful}

We use empirical Bernstein bounds, properties of exponential mechanism
and Johnson-Lindenstrauss lemmas which we state below.

\begin{lem}[Relative deviation bound ]\label{lem:rel_bound}
For any hypothesis set $\sH$ of functions mapping from $\cX$ to $R$,
with probability at least $1 - \beta$, the following inequality holds
for all $h \in \sH$:
\begin{align*}
R_{\sD}(h) \leq \widehat R_S(h)
+  2\sqrt{\widehat R_S(h)\frac{ V(\sH)\log (2m) + 
\log \frac{4}{\beta}}{m}}
+ 4 \frac{V(\sH) \log (2m) + \log\frac{4}{\beta}}{m},
\end{align*}
where $V(\sH)$ is the VC-dimension of class $\sH$.
\end{lem}
The above lemma is obtained by combining \cite[Corollary 7]{cortes2019relative} and VC-dimension bounds.
\begin{lem}[Relative deviation margin bound \cite{cortes2021relative}]\label{lem:rel_margin_bound}
Fix $\rho \geq 0$. Then, for any hypothesis set $\sH$ of functions
mapping from $\cX$ to $\Rset$ with $d = \fat_{\frac{\rho}{16}}(\sH)$,
with probability at least $1 - \beta$, the following holds for all
$h \in \sH$:
\begin{align*}
R_{\sD}(h) \leq \widehat R^\rho_S(h) + 2\sqrt{\widehat R^\rho_S(h) \frac{\fatshatteringbound}{m}}  
+ \frac{ \fatshatteringbound}{m},
\end{align*}
where $\fatshatteringbound =  1 + d \log_2(2c^2m) \log_2\frac{2cem}{d} + \log \frac{1}{\beta}$ and $c = 17$.
\end{lem}

\thnote{Changed the following results to Raef's suggestion. Removed the name JL lemma as its not the standard JL lemma anymore.}

\begin{lem}
\label{lem:JL}
Let $\beta, \gamma \in (0, 1)$. Let \mbox{$T \subset \Rset^d$} be any
set of $m$ vectors. There exists
$k=O\paren[\bigg]{\frac{\log\paren*{\frac{m}{\beta}}}{\gamma^2}}$ such
that for any random $k\times d$ matrix $\Phi$ with entries drawn
i.i.d. uniformly from $\{\pm \frac{1}{\sqrt{k}}\}$, the following
inequalities hold simultaneously with probability at least $1 - \beta$
over the choice of $\Phi$:
\begin{itemize}
    \item For any $u\in T$,
\begin{align*}
    \left(1-\frac{\gamma}{3}\right)\norm{u}_2^2 \leq  &\norm{\Phi u}_2^2 \leq \left(1+\frac{\gamma}{3}\right)\norm{u}_2^2.
   \end{align*}
\item For any $u, v\in T$,
\begin{align*}
 \lvert \langle \Phi u, \Phi v \rangle &- \langle u, v \rangle\rvert \leq \frac{\gamma}{3}\,\norm{u}_2\,\norm{v}_2.
\end{align*}
\end{itemize}
\end{lem}
\begin{proof}
The first property is simply the Johnson-Lindenstrauss (JL) property
and follows from the standard JL lemma (see, e.g.,
\cite{johnson1984extensions,larsen2017optimality}). Below we show both first and second property holds simultaneously. 
    Define $$\widetilde{T}\triangleq
\left\{z\in\Rset^d:~ z=\frac{u}{\norm{u}_2}\pm \frac{v}{\norm{v}_2},~ u,
v\in T\right\}.$$ Note that the number of non-zero vectors in
$\widetilde{T}$ is at most $m^2$. By the JL lemma
\cite{johnson1984extensions,larsen2017optimality} over the set $T\cup
\widetilde{T}$, there exists
$k=O\left(\frac{\log(m^2/\beta)}{\gamma^2}\right)=O\left(\frac{\log\paren*{\frac{m}{\beta}}}{\gamma^2}\right)$
and $\Phi'$ such that with probability $\geq 1-\beta$, for all $u, v
\in T$ we have
\begin{align}
    \left(1-\frac{\gamma}{3}\right)\norm{\Bar{u}+\Bar{v}}_2^2 \leq  &\norm{\Phi'(\Bar{u}+\Bar{v})}_2^2 \leq \left(1+\frac{\gamma}{3}\right)\norm{\Bar{u}+\Bar{v}}_2^2, \label{ineq:conseq-JL1}\\
    \left(1-\frac{\gamma}{3}\right)\norm{\Bar{u}-\Bar{v}}_2^2 \leq  &\norm{\Phi'(\Bar{u}-\Bar{v})}_2^2 \leq \left(1+\frac{\gamma}{3}\right)\norm{\Bar{u}-\Bar{v}}_2^2, \label{ineq:conseq-JL2} \\
     \left(1-\frac{\gamma}{3}\right)\norm{u}_2^2 \leq  &\norm{\Phi' u}_2^2 \leq \left(1+\frac{\gamma}{3}\right)\norm{u}_2^2 \label{ineq:orig-JL}.
   \end{align}
where $\Bar{u}=\frac{u}{\norm{u}_2}$ and $\Bar{v}=\frac{v}{\norm{v}_2}$. \eqref{ineq:orig-JL} implies the first result in the lemma. Now, fix any $u, v \in T$. Let $\Bar{u}=\frac{u}{\norm{u}_2}$ and $\Bar{v}=\frac{v}{\norm{v}_2}$. Observe that for any $\mathbf{a}, \mathbf{b}\in \Rset^d$, we have $\langle \mathbf{a}, \mathbf{b}\rangle = \frac{1}{4}\left(\norm{\mathbf{a}+\mathbf{b}}_2^2-\norm{\mathbf{a}-\mathbf{b}}_2^2\right)$. Hence, we have 
\begin{align*}
    \big\lvert \langle\Bar{u}, \Bar{v}\rangle - \langle \Phi'\Bar{u}, \Phi'\Bar{v}\rangle \big\rvert &\leq \frac{1}{4}\big\lvert\norm{\Bar{u}+\Bar{v}}^2_2-\norm{\Phi'(\Bar{u}+\Bar{v})}^2_2\big\rvert + \frac{1}{4}\big\lvert\norm{\Phi'(\Bar{u}-\Bar{v})}^2_2-\norm{\Bar{u}-\Bar{v}}^2_2\big\rvert
    \tag{\mnote{I think the squares are missing for all the terms; also, you can use the macro "abs"} \thnote{Yes you are right. Fixed the squares.}}\\
    &\leq \frac{\gamma}{12}\left(\norm{\Bar{u}+\Bar{v}}^2_2+\norm{\Bar{u}-\Bar{v}}^2_2\right)\\
    &\leq \frac{\gamma}{3},
\end{align*}
where the second inequality follows from (\ref{ineq:conseq-JL1}) and
(\ref{ineq:conseq-JL2}) 'and the third inequality follows from the
triangle inequality and the fact that
$\norm{\Bar{u}}_2=\norm{\Bar{v}}_2=1$.  Hence, we finally have
$$\lvert \langle u, v\rangle - \langle \Phi' u, \Phi' v\rangle \rvert \leq \norm{u}_2\norm{v}_2 \,\lvert \langle\Bar{u}, \Bar{v}\rangle - \langle \Phi'\Bar{u}, \Phi'\Bar{v}\rangle \rvert \leq \frac{\gamma}{3}\norm{u}_2\norm{v}_2.$$
\end{proof}


The time complexity to apply the random matrix $\Phi$ in
Lemma~\ref{lem:JL} to a vector $v$ is $k \cdot d$, which can be
prohibitive in many cases. There are several works which provide
$\Phi$ that support fast matrix vector
products. \cite{ailon2006approximate} provides a $\Phi$ which can be
applied in time $c' d \log d + \frac{c'\log(d/\beta)
  \log^2(1/\beta)}{\gamma^2}$, however the results are stated with
constant probability. \cite{nelson2010johnson} gave a slightly
different construction which can be applied in time $c' d \log d +
\frac{c'\log(d/\beta) \log^2(1/\beta)}{\gamma^4}$ and the results hold
with high probability. \cite{ailon2009fast} provided a construction
which can be applied in time $c' d \log k$ for $k =
O(d^{0.499})$. \cite{krahmer2011new} showed that any RIP matrix can
be used for JL-transform and provided JL-transform results for several
fast random projections. Since we need high probability bounds without
any restrictions, we use the following result, which is
computationally efficient, but is suboptimal in the projection
dimension up to logarithmic factors.

\begin{lem}\label{lem:FJL}
Let $\beta, \gamma \in (0, 1)$ and $c$ and $c'$ be sufficiently large
constants. Let \mbox{$T \subset \Rset^d$} be any set of $m$ vectors. Let
$k=\frac{c \log\paren*{\frac{m}{\beta}} \log \paren*{\frac{m}{\gamma
  \beta}}}{\gamma^2}$. There exists a matrix $\Phi$ which can be
applied to any vector $v$ in time $c' d \log d + c' k$, such that the
following inequalities hold simultaneously with probability at least
$1 - \beta$ over the choice of $\Phi$:
\begin{itemize}
    \item For any $u\in T$,
\begin{align}
\label{eq:fast-JL}
    \left(1-\frac{\gamma}{3}\right)\norm{u}_2^2 \leq  &\norm{\Phi u}_2^2 \leq \left(1+\frac{\gamma}{3}\right)\norm{u}_2^2.
   \end{align}
\item For any $u, v\in T$,
\begin{align}
 \lvert \langle \Phi u, \Phi v \rangle &- \langle u, v \rangle\rvert \leq \frac{\gamma}{3}\,\norm{u}_2\,\norm{v}_2.
 \label{eq:fast-JL-cor}
\end{align}
\end{itemize}
\end{lem}
Property \eqref{eq:fast-JL} in the above result follows directly from \cite{nelson2015dimensionality} and the proof for property~\eqref{eq:fast-JL-cor} is similar to that of Lemma~\ref{lem:JL} and is omitted.

\section{DP Algorithms for Linear Classification with Margin Guarantees}
\label{app:linear}

\subsection{Proof of Theorem~\ref{thm:prv_mrg_class}}\label{app:prv_mrg_class_proof}
\PrvMrgClass*

\begin{proof}
The proof of privacy follows from combining the following two
properties: $\tilde{w}$ is generated via the exponential mechanism,
which an $\varepsilon$-differentially private mechanism, and $\Phi$ is
generated independently of $S$.

We now prove the accuracy guarantee of
Algorithm~\ref{alg:prvmrg}.     If $m <
\frac{c\Lambda^2r^2\log\paren*{\frac{m}{\beta}}\log\left(\frac{20\Lambda
    r}{\beta \rho}\right)}{\rho^2\varepsilon}$, for some constant $c$, the bound follow trivially. Hence in the rest of the proof we assume
that $m$ is at least
$\frac{c\Lambda^2r^2\log\paren*{\frac{m}{\beta}}\log\left(\frac{20\Lambda
    r}{\beta \rho}\right)}{\rho^2\varepsilon}$ for some large constant $c$.  Let 
    \[
    \alpha =
 \frac{c\Lambda^2r^2\log\paren*{\frac{m}{\beta}}\log\left(\frac{20\Lambda
     r}{\beta \rho}\right)}{\rho^2\varepsilon m}.
     \]
First, observe that
\begin{align}
    \widehat R_S(w^\prv)&=\frac{1}{m}\sum_{(x, y)\in S}\ind\left(y\langle w^\prv, x\rangle\right)\nonumber\\
    &=\frac{1}{m}\sum_{(x, y)\in S}\ind\left(y\langle \Phi^{\top}\tilde{w}, x\rangle\right)\nonumber\\
    &=\frac{1}{m}\sum_{(x, y)\in S}\ind\left(y\langle \tilde{w}, \Phi x\rangle\right)\nonumber\\
    &=\frac{1}{m}\sum_{(x_\Phi, y)\in S_\Phi}\ind\left(y\langle \tilde{w}, x_\Phi\rangle\right)\nonumber\\
    &=\widehat R^{(k)}_{S_\Phi}(\tilde{w}). \nonumber
\end{align}

Let $\widehat{w}\in\argmin\limits_{w\in \cC}\widehat R^{(k)}_{S_\Phi}(w)$. Note that $\lvert\cC\rvert=\left(\frac{20\, \Lambda \,r}{\rho}\right)^{k}$, where $k=\frac{c'\,\Lambda^2\,r^2\,\log\paren*{\frac{m}{\beta}}}{\rho^2}$. Hence, by the accuracy properties of the exponential mechanism and the fact that $m\geq \frac{4\log\left(\frac{4\lvert \cC\rvert}{\beta}\right)}{\varepsilon\alpha}$, we have that with probability at least $ 1-\beta/4$, \thnote{@Raef: to keep the paper inclusive, would it be possible to add this property of the exponential paper in the appendix along with JL-transform lemmas?}\rnote{This a very standard and basic result for the exponential mechanism. We can simply cite the very first paper by McSherry-Talwar or even the textbook.}
\begin{align}
    \widehat R^{(k)}_{S_\Phi}(\tilde{w})&\leq \widehat R^{(k)}_{S_\Phi}(\widehat{w})+\alpha.\nonumber
\end{align}
Combining the above facts, we get that with probability at least $
1-\beta/4$,
\begin{align}
    \widehat R_S(w^\prv)&\leq \widehat R^{(k)}_{S_\Phi}(\widehat{w})+\alpha. \label{ineq:wprv-to-w-hat}
\end{align}
Let $w^\ast\in \argmin\limits_{w\in \B^d(\Lambda)}\widehat R^\rho_S(w)$ and let $w^\ast_\Phi=\Phi w^\ast$. Note that 
\[
\norm{w^\ast_\Phi}_2 \leq \sqrt{1+\frac{\rho}{3\Lambda r}} \norm{w^\ast}_2 
\leq 2 \norm{w^\ast}_2 \leq 2\Lambda,
\]
and hence $w^\ast_\Phi \in \B^k(2\Lambda)$. Since $\cC$ is $\frac{\rho}{10\,r}$-cover of $\B^k(2\Lambda)$, then there must be $w_c\in \cC$ such that $\norm{w_c-w^\ast_\Phi}_2\leq \frac{\rho}{10\,r}$. Hence, observe that for any $(x_\Phi, y)\in S_\Phi$, 
\begin{align*}
    y\langle w_c, x_\Phi\rangle &= y\langle w^\ast_\Phi, x_\Phi\rangle + y\langle w_c-w^\ast_\Phi, x_\Phi\rangle\\
    &\geq y\langle w^\ast_\Phi, x_\Phi\rangle - \norm{w_c-w^\ast_\Phi}_2\,\norm{x_\Phi}_2\\
    &\geq y\langle w^\ast_\Phi, x_\Phi\rangle - \frac{\rho}{10 \,r}\norm{x_\Phi}_2.
    \end{align*}
Now, by Lemma~\ref{lem:JL}, with probability\mnote{Typo: I think there are several instances of that.}\rnote{Could you clarify what typo you are referring to?} at least $ 1-\beta/4$, for all $x_\Phi$ s.t. $(x_\Phi, y)\in S_\Phi$ we have $\norm{x_\Phi}_2\leq \sqrt{1+\frac{\rho}{3\Lambda r}}\norm{x}_2\leq \sqrt{1+\frac{\rho}{3\Lambda r}}r$. Hence, we get that with probability at least $ 1-\beta/4$ for all $(x_\Phi, y)\in S_\Phi$
\begin{align*}
    y\langle w_c, x_\Phi\rangle &\geq y\langle w^\ast_\Phi, x_\Phi\rangle - 0.15\rho.
\end{align*}
The last inequality implies that for any $\rho'>0$, with probability at least $ 1-\beta/4$ (over the choice of $\Phi$), we must have $\widehat R^{\rho', (k)}_{S_\Phi}(w_c)\leq \widehat R^{\rho' + 0.15\rho, (k)}_{S_\Phi}(w^\ast_\Phi)$. In particular, with probability at least $ 1-\beta/4$ we have 
\begin{align}
    \widehat R^{0.5\rho, (k)}_{S_\Phi}(w_c)&\leq \widehat R^{0.65\rho, (k)}_{S_\Phi}(w^\ast_\Phi).\label{ineq:wc-to-w-star-phi}
\end{align}
\rnote{The superscript with the comma here $ \widehat R^{0.5\rho, (k)}$ is confusing. Looks like as if $\rho$ is a function of $k$ and the comma is a typo.} Moreover, by the definition of $\widehat{w}$, we have $\widehat R^{(k)}_{S_\Phi}(\widehat{w})\leq \widehat R^{(k)}_{S_\Phi}(w_c)\leq \widehat R^{0.5\rho, (k)}_{S_\Phi}(w_c)$. 
\thnote{Why do we want to use the last inequality here? I agree it doesn't matter much more than constants, but I did not get its significance.}\rnote{I agree. There is no significance to $\rho/2$ here (that was sort of an arbitrary choice from my end, which is admittedly confusing)} Combining this fact with (\ref{ineq:wprv-to-w-hat}) and (\ref{ineq:wc-to-w-star-phi}), we get that with probability at least $ 1-\beta/2$
\begin{align}
    \widehat R_S(w^\prv)&\leq \widehat R^{0.65\rho, (k)}_{S_\Phi}(w^\ast_\Phi)+\alpha.\label{ineq:wprv-to-w-star-phi}
\end{align}
Now, by Lemma~\ref{lem:JL} and the fact that $\norm{w^\ast}_2\leq \Lambda$ and $\norm{x}_2\leq r$, it follows that with probability at least $ 1-\beta/4$ for all $(x_\Phi, y)\in S_\Phi$, we have  
$$y\langle w^\ast, x\rangle \geq \rho \implies y\langle w^\ast_\Phi, x_\Phi\rangle \geq \rho/3.$$
This directly implies that with probability at least $ 1-\beta/4$, 
\begin{align*}
    \widehat R^{0.65\rho, (k)}_{S_\Phi}(w^\ast_\Phi)&\leq \widehat R^\rho_S(w^\ast).
\end{align*}
Combining this with (\ref{ineq:wprv-to-w-star-phi}), we can assert that with probability at least $ 1-\frac{3}{4}\beta$, we have \begin{align}
    \widehat R_S(w^\prv)&\leq \widehat R^{\rho}_S(w^\ast)+\alpha.\label{ineq:wprv-to-w-star}
\end{align}
In the final step of the proof, we rely on a standard uniform convergence argument to bound $R_{\sD}(w^\prv)$ in terms of $\widehat R_S(w^\prv)$. Note that the VC-dimension of $\left\{\sgn\circ h_{w}: w\in \cC\right\}$ is $k$. By Lemma~\ref{lem:rel_bound}, with probability at least $1 -\beta/4$
\begin{align*}
 R_{\sD}(w^\prv) - \widehat R_S(w^\prv)
\leq   2\sqrt{\widehat R_S(w^\prv) \frac{k \log(2m) + \log
  ({16}/{\beta})}{m}}
+ 4 \frac{k\log(2m) + \log ({16}/{\beta})}{m}.
\end{align*}
Combining the above two equations, we get  with probability at least $1 - \beta$,
\begin{align*}
  R_{\sD}(w^\prv) 
  \leq \widehat R^{\rho}_S(w^\ast)
  + 2 \sqrt{\widehat R^{\rho}_S(w^\ast) \frac{k\log(2m) + \log
  ({16}/{\beta})}{m}}
 + 2\alpha +  8\frac{ k\log (2m) + \log ({16}/{\beta})}{m}. 
\end{align*}
The lemma follows from observing that  $w^\ast\in \argmin\limits_{w\in \B^d(\Lambda)}\widehat R^\rho_S(w)$ if and only if
$w^\ast\in \argmin\limits_{w\in \B^d(\Lambda)} \widehat R^{\rho}_S(w)
  + 2 \sqrt{\widehat R^{\rho}_S(w) \frac{k\log (2m) + \log
  ({16}/{\beta})}{m}}$.
\end{proof}

\subsection{Algorithm~\ref{alg:nsgd} of Section~\ref{sec:effprvmargin} and Proof of Lemma~\ref{lem:eff-prv-erm}}
\label{app:algs}
Here, we give the details of Algorithm~\ref{alg:nsgd} invoked in step~\ref{stp:dp-erm} of Algorithm~\ref{alg:effprvmrg}. We describe here a more general setup where the loss function is any (possibly non-smooth) convex generalized linear loss (GLL). Given a parameter space $\cW$, feature space $\cX$, and label/target set $\cY$, a GLL is a loss function defined over $\cW\times (\cX \times \cY)$ that can be written as $\ell(\langle w, x\rangle, y),~ w\in \cW, x\in \cX, y\in \cY$ for some function $\ell:\Rset \times \cY\rightarrow \Rset$. Here, we assume that for any $y\in\cY$, $\ell(\cdot, y)$ is convex and $\frac{1}{\rho}$-Lipschitz. We also assume that 
$\cW\subseteq\B^k(\Lambda)$ for some $\Lambda >0$, $\cX \subseteq \B^k(\tilde r)$ for some $\tilde r>0$, and $\cY\subseteq [-1, 1]$. Given a dataset $\wt S = ((x_1, y_1), \ldots, (x_m, y_m)) \in (\cX \times \cY)^m$, we define the empirical risk of $w\in \cW$ with respect to $\wt S$ as $\h L_{\wt S}(w)\triangleq \frac{1}{m}\sum_{i=1}^m \ell(\langle w, x_i\rangle, y)$. Note that the setup in Algorithm~\ref{alg:effprvmrg} is a special case of the above.

Given an input dataset $\wt S \in (\cX\times \cY)^m$,  Algorithm~\ref{alg:nsgd} below invokes the ``Phased SGD algorithm for GLL'' \cite[Algorithm~2]{bassily2021differentially} on a set $\h S$ of $m$ samples drawn uniformly with replacement from $\wt S$, and hence obtain an output $\tilde{w}\in \cW$. In the sequel, we will refer to the algorithm in \cite{bassily2021differentially} as $\cA_{\sf GLL}$. Note that the expected loss with respect to the choice of $\h S\leftarrow \wt S$ is the empirical risk with respect to $\wt S$. Hence, one can derive a bound on the expected excess empirical risk that is roughly the same as the bound in \cite[Theorem~6]{bassily2021differentially} on the expected excess population risk. However, we note that ensuring that Algorithm~\ref{alg:nsgd} is $(\varepsilon, \delta)$-DP does not follow directly from the privacy guarantee of the $\cA_{\sf GLL}$ since the sample $\h S$ may contain duplicate entries from $\wt S$. Nonetheless, we show that the privacy guarantee can be attained by appropriately setting the input privacy parameters to $\cA_{\sf GLL}$ together with a careful privacy analysis. 
To transform the in-expectation bound into a high-probability bound, we perform a standard confidence-boosting procedure \cite[Appendix~D]{BassilySmithThakurta2014}, where the procedure described above is repeated independently $M=O(\log(1/\beta))$ times to generate $\tilde{w}^1, \ldots, \tilde{w}^M$, and finally, the exponential mechanism (with a score function $-\h L_{\wt S}$) is used to privately select a final output $\tilde{w}\in \{\tilde{w}^1, \ldots, \tilde{w}^M\}$.


\begin{algorithm}
	\caption{DP-ERM algorithm for GLLs}
	\begin{algorithmic}[1]
		\REQUIRE Private dataset $\wt S= \big((x_1, y_1),\dots, (x_m, y_m)\big) \in (\cX\times\cY)^m$, where $\cX\subseteq \B^k(\tilde r)$ and $\cY\subseteq [-1, 1]$; parameter space $\cW\subseteq \B^k(\Lambda)$; privacy parameters $(\varepsilon, \delta)$; confidence parameter $\beta \in (0, 1)$; convex, $\frac{1}{\rho}$-Lipschitz loss function $\ell$ for some $\rho >0$; Oracle access to algorithm $\cA_{\sf GLL}$ \cite[Algorithm~2]{bassily2021differentially}.
\STATE Let $M:= \log(2/\beta)$.
\STATE Let $\varepsilon':= \frac{\varepsilon}{4M\log(2M/\delta)}$.
\STATE Let $\delta' := \frac{\delta^2}{4M\log(2M/\delta)}$.

\FOR{$t=1$ to $M$}
\STATE Sample $\h S^t=\big((\h x^t_1, \h y^t_1), \ldots \big(\h x^t_m, \h y^t_m\big)\big) \leftarrow \wt S$ uniformly with replacement. \label{step:samp-w-rep}

\STATE  $\tilde w^t=\cA_{\sf GLL}(\h S^t; \varepsilon', \delta')$, where $\cA_{\sf GLL}$ is \cite[Algorithm~2]{bassily2021differentially} (the other obvious inputs to $\cA_{\sf GLL}$ are omitted; the smoothing parameter and the oracle accuracy parameter of $\cA_{\sf GLL}$ are set as in \cite[Theorem~6]{bassily2021differentially}). 
\ENDFOR

\STATE Run the exponential mechanism with privacy parameter $\varepsilon/2$ to select $\tilde{w}$ from the set $\big(\tilde{w}^1, \ldots, \tilde{w}^M\big)$ associated with scores $\big(-\h L_{\wt S}(\tilde{w}^t):~t\in [M]\big)$. \label{step:conf-boost-exp}
\RETURN $\tilde{w}$
	\end{algorithmic}
	\label{alg:nsgd}
\end{algorithm}

\effprverm*

\begin{proof}
First, we show the privacy guarantee. Fix a round $t\in [M]$ of Algorithm~\ref{alg:nsgd}. We will show that the $t$-th round is $(\frac{\varepsilon}{2M}, \frac{\delta}{M})$-DP. Suppose we can do that. Then, by the basic composition property of DP, the entire $M$ rounds of the algorithm is $(\frac{\varepsilon}{2}, \delta)$-DP. Next, we note that step~\ref{step:conf-boost-exp} is $(\frac{\varepsilon}{2}, 0)$-DP by the privacy guarantee of the exponential mechanism. Hence, again by basic composition of DP, we conclude that Algorithm~\ref{alg:nsgd} is $(\varepsilon, \delta)$-DP. Thus, it remains to show that for any fixed $t\in [M]$, the $t$-th round is $(\hat\varepsilon, \hat\delta)$-DP, where $\hat\varepsilon=\frac{\varepsilon}{2M}$ and $\hat\delta=\frac{\delta}{M}$. Fix any data point $(x_i, y_i)\in \wt S$. Let $J$ denote the number of appearances of $(x_i, y_i)$ in $\h S^t$. Note that $J\sim \mathsf{Bin}(m, 1/m)$. Hence, using the multiplicative Chernoff's bound, $J\leq 2\log(2/\hat\delta)$ with probability at least $1-\hat\delta/2$. We will show that conditioned on this event, the $t$-th round is $(\hat\varepsilon, \hat\delta/2)$-DP, which suffices to prove our privacy claim for round $t$. Given that $J\leq 2\log(2/\hat\delta)$ and since $\cA_{\sf GLL}$ is $(\frac{\hat\varepsilon}{2\log(2/\hat\delta)}, \frac{\delta\hat\delta}{4\log(2/\hat\delta)})$-DP with respect to to its input dataset $\h S^t$, then by the group-privacy property of DP \cite{DworkRoth2014}, round $t$ is $(\hat\varepsilon, \delta'')$-DP with respect to the dataset $\wt S$, where 
\begin{align*}
 \delta''&=\frac{\delta\hat\delta}{4\log(2/\hat\delta)}\cdot\sum_{j=0}^{2\log(2/\hat\delta)}e^{\frac{\hat\varepsilon}{2\log(2/\hat\delta)}\, j}\\
    &=\frac{\delta\hat\delta}{4\log(2/\hat\delta)}\cdot\frac{e^{\hat\varepsilon}-1}{e^{\frac{\hat\varepsilon}{2\log(2/\hat\delta)}}-1}\\
    &\leq \frac{\delta\hat\delta\paren*{e^{\hat\varepsilon}-1}}{2\hat\varepsilon}\\
    &\leq \frac{\hat\delta}{2},
\end{align*}
where the third inequality follows from the fact that $e^{\frac{\hat\varepsilon}{2\log(2/\hat\delta)}}-1\geq \frac{\hat\varepsilon}{2\log(2/\hat\delta)}$, and the last step follows from the fact that $\frac{e^a-1}{a}$ is increasing in $a>0$ and the assumption that $\varepsilon \leq \log(1/\delta)$ (and hence $\hat\varepsilon < \varepsilon \leq \log(1/\delta)$). Hence, we have shown that any given round of the algorithm is $(\varepsilon, \hat\delta)$-DP. This concludes the proof of the privacy guarantee. 

We now prove the bound on the excess empirical risk. Fix any round $t$. Let $\h \sD_{\wt S}$ denote the empirical distribution of $\wt S$. Note that $\h S^t\sim \h{\sD}_{\wt S}^m$, i.e., $\h S^{t}$ is comprised of $m$ independent samples from $\h \sD_{\wt S}$. Hence, $\ex{(x, y)\sim \h \sD_{\wt S}}{\ell(\langle w, x\rangle, y)}=\h L_{\wt S}(w)$. Thus, by the excess risk guarantee of $\cA_{\sf GLL}$ \cite[Theorem~6]{bassily2021differentially}, we have 
\begin{align*}
\ex{}{\emphng^{\rho}_{\wt S}(\tilde w)}-\min\limits_{w\in \B^k(
        \Lambda)}\emphng^{\rho}_{ \wt S}(w) &=  \frac{\Lambda \tilde r}{\rho}\cdot O\paren*{\frac{1}{\sqrt{m}}+\frac{\sqrt{k\log(\frac{1}{\delta'})}}{\varepsilon' m}}\\
    &= \frac{\Lambda \tilde r}{\rho}\cdot O\paren*{\frac{1}{\sqrt{m}}+\frac{\sqrt{k}\log^{\frac{3}{2}}(\frac{1}{\delta})\log(\frac{1}{\beta})}{\varepsilon m}},
\end{align*}
where the expectation is with respect to the sampling step
(step~\ref{step:samp-w-rep} of Algorithm~\ref{alg:nsgd}) and the
randomness in $\cA_{\sf GLL}$. Note the last step follows from the
setting of $\varepsilon'$ and $\delta'$ in Algorithm~\ref{alg:nsgd}
and the fact that
$\log\paren*{\log(1/\beta)/\delta}=O(\log(1/\delta))$, which follows
from the assumption $\delta < 1/m$ (in the statement of the lemma) and
$\log(1/\beta)< m$ (since the bound would be trivial otherwise). Given
this expectation guarantee on the output of each round, the final
selection step (step~\ref{step:conf-boost-exp}) returns a parameter
$\tilde{w}$ that satisfies the bound above with probability at least
$1-\beta$. This can be shown by following the same argument in
\cite[Appendix~D]{BassilySmithThakurta2014} while noting that the
sensitivity of the score function $-\h L_{\wt S}$ is bounded by
$\frac{\Lambda \tilde r}{m}$.

Finally, the running time of $\cA_{\sf GLL}$, measured in terms of
gradient computations, is $O(m\log(m))$
\cite[Theorem~6]{bassily2021differentially}. Hence, the gradient
complexity of Algorithm~\ref{alg:nsgd} is bounded by
$O(m\log(m)\log(1/\beta))$.
\end{proof}

\subsection{Proof of Theorem~\ref{thm:prv_mrg_hinge}}\label{app:prv_mrg_hinge_proof}

\PrvMrgHinge*

\begin{proof}
  The proof of privacy follows directly from the
  $(\varepsilon, \delta)$-DP guarantee of Algorithm~\ref{alg:nsgd}
  (step~\ref{stp:dp-erm}) and the fact that DP is closed under
  post-processing. 

  Next, we will prove the claimed margin bound. For simplicity and
  without loss of generality, we will set $\Lambda=1$. For the general
  setting of $\Lambda$, the claimed bound follows by rescaling the
  parameter vectors in the proof.

  Recall that $x_\Phi\triangleq \Phi x$. Let
  $w^\ast\in \argmin\limits_{w\in \B^d} \emphng^{\rho}_S(w)$. Define
  $w^\ast_\Phi\triangleq \Phi w^\ast$. By Lemma~\ref{lem:FJL}, there
  is
  $\gamma =
  O\paren*{\sqrt{\frac{\log\paren*{\frac{m}{\beta}}}{k}}}=O\paren*{\sqrt{\frac{\log(1/\beta)}{\varepsilon
        m}}\log^{\frac{3}{4}}(1/\delta)}$ such that with probability
  at least $ 1-\beta/3$ over the randomness of $\Phi$, for every
  feature vector $x$ in the training set $S$, we have
\begin{align}
   \norm{x_\Phi}_2^2&\leq \paren*{1+\frac{\gamma}{3}}\norm{x}_2^2\label{ineq:prop-FJL-1}\\
    1- \frac{y \langle w^\ast_\Phi , x_\Phi\rangle}{\rho} &\leq 1 - \frac{y \langle w^\ast , x\rangle}{\rho} + \frac{r\gamma}{\rho}\label{ineq:prop-FJL-2}
\end{align}
We condition on this event for the remainder of the
proof. 

Note that (\ref{ineq:prop-FJL-1}) implies that 
$$S_{\Phi}=\{(x_\Phi, y): (x, y)\in S\};$$ 
that is, for all feature vectors $x$ in the dataset $S$, $x_\Phi\in \B^k(2r)$ (i.e., $\Pi_{\B^k(2r)}(x_\Phi)=x_\Phi$). 

Let $\sD_\Phi$ denote the distribution of the pair $(x_\Phi, y)$,
where $(x, y)\sim \sD$. Via a standard margin bound \cite[Theorems~5.8
\& 5.10]{mohri2018foundations}, with probability at least $1-\beta/3$ over the
choice of the training set $S$, we have
\begin{align*}
    \forall w \in \B^k~\quad~ R_{\sD_\Phi}(w)&\leq \emphng^{\rho}_{S_\Phi}(w) + \frac{4 r}{\rho \sqrt{m}}+2 \sqrt{\frac{\log(6/\beta)}{m}}
\end{align*}
It follows that with probability at least $ 1-\beta/3$, we have 
\begin{align*}
    R_{\sD_\Phi}(\tilde{w})&\leq \emphng^{\rho}_{S_\Phi}(\tilde{w}) + \frac{4 r}{\rho \sqrt{m}}+2 \sqrt{\frac{\log(6/\beta)}{m}},
\end{align*}
where $\tilde{w}$ is the output of step~\ref{stp:dp-erm} of
Algorithm~\ref{alg:effprvmrg}. Moreover, note that
\begin{align*}
    R_{\sD}(w^\prv)&= \pr{(x, y)\sim \sD}{y\langle w^\prv, x\rangle \leq 0}\\
    &= \pr{(x_\Phi, y)\sim \sD_\Phi}{y\langle \tilde{w}, x_\Phi\rangle \leq 0}\\
    &=R_{\sD_\Phi}(\tilde{w})
\end{align*}
Thus, we get that with probability at least $ 1-\beta/3$, 
\begin{align}
    R_{\sD}(w^\prv)&\leq \emphng^{\rho}_{S_\Phi}(\tilde{w}) + \frac{4 r}{\rho \sqrt{m}}+2 \sqrt{\frac{\log(6/\beta)}{m}}. \label{eq:effprvmrg-proof-1}
\end{align}
    

By Lemma~\ref{lem:eff-prv-erm}, with probability at least $1-\beta/3$ over the
randomness of Algorithm~\ref{alg:nsgd} (step~\ref{stp:dp-erm} of
Algorithm~\ref{alg:effprvmrg}), we have that
\begin{align}
    \emphng^{\rho}_{S_\Phi}(\tilde{w})&\leq \emphng^{\rho}_{S_\Phi}(\widehat{w}) + \frac{\Lambda r}{\rho}\left(\frac{1}{\sqrt{m}}+\frac{\sqrt{\log(\frac{m}{\beta})\log(\frac{1}{\beta})}\log^{\frac{3}{4}}(\frac{1}{\delta})}{\sqrt{\varepsilon m}}\right), 
  \label{eq:effprvmrg-proof-2}
\end{align}
where $\widehat{w}\in \argmin\limits_{w\in \B^k}\emphng^{\rho}_{S_\Phi}(w)$.  
Moreover, (\ref{ineq:prop-FJL-2}) implies 
\begin{align*}
    \emphng^{\rho}_{S_\Phi}(w^\ast_\Phi)&\leq \emphng^{\rho}_{S}(w^\ast) + \frac{r}{\rho}\cdot O\paren*{\sqrt{\frac{\log(1/\beta)}{\varepsilon m}}\log^{\frac{3}{4}}(1/\delta)}
\end{align*}
Note that by definition of $\widehat{w}$, we have $\emphng^{\rho}(\widehat{w}; S)\leq \emphng^{\rho}_{S_\Phi}(w^\ast_\Phi)$. Hence, we have 
\begin{align}
    \emphng^{\rho}_{S_\Phi}(\widehat{w})&\leq \emphng^{\rho}_{S}(w^\ast) + \frac{r}{\rho}O\paren*{\sqrt{\frac{\log(1/\beta)}{\varepsilon m}}\log^{\frac{3}{4}}(1/\delta)} \label{eq:effprvmrg-proof-3}
\end{align}
Now, by combining (\ref{eq:effprvmrg-proof-1}),
(\ref{eq:effprvmrg-proof-2}), and (\ref{eq:effprvmrg-proof-3}), we
reach the desired bound.

Finally, concerning the running time, observe that the Fast
JL-transform (steps~\ref{step:FJL1} and \ref{step:fast-JL}) takes
$O(md\log(d)+\varepsilon m^2\log(m)/\log^{3/2}(1/\delta))$ (follows
from Lemma~\ref{lem:FJL}), the DP-ERM algorithm
(Algorithm~\ref{alg:nsgd}) invoked in step~\ref{stp:dp-erm} has $O(m)$
gradient steps; each of which takes involves
$O(k+\log(m))=O(\varepsilon m \log(m)/\log^{3/2}(1/\delta))$
operations. That is, the total number of operations of this step is
$O(\varepsilon m^2 \log(m)/\log^{3/2}(1/\delta))$. Finally, the
step~\ref{step:output} requires
$O(d k)=O(\varepsilon dm\log(m)/\log^{3/2}(1/\delta))$. Thus, the
overall running time is
$O\paren*{m d\log\paren*{\max(d, m)}+\varepsilon
  m^2\log(m)/\log^{\frac{3}{2}}(1/\delta)}$.

\end{proof}

\section{DP Algorithms for
Kernel-Based Classification with Margin Guarantees}
\ignore{
\section{Margin guarantees for kernel-based classification}
}
\label{app:kernels}

\PrivKernel*

\begin{proof}
  Let $S_{\hpsi}$ be as defined in step~\ref{step:S_hpsi} in
  Algorithm~\ref{alg:prvkermrg}. Note that by
  Theorem~\ref{thm:RFF-approx}, we have $\norm{\hpsi(x_i)}_2=r$ for
  all $i\in [m]$. Moreover, the output of
  Algorithm~\ref{alg:prvkermrg} depends only on $S_{\hpsi}$. Thus, the
  privacy guarantee follows directly from the privacy guarantee of
  Algorithm~\ref{alg:effprvmrg} (Theorem~\ref{thm:prv_mrg_hinge}).

  Next, we turn to proving the claimed margin bound. First, note that
  using the margin bound attained by Algorithm~\ref{alg:effprvmrg}
  (Theorem~\ref{thm:prv_mrg_hinge}), it follows that for any fixed
  realization of the randomness in $\hpsi$, with probability $1-\beta/2$ over the
  choice of $S\sim \cD^m$ and the internal randomness of
  Algorithm~\ref{alg:effprvmrg}, we have
\begin{align}
    R_\sD(h^{\hpsi}_{w^\prv})&\leq \min\limits_{w\in \B^{2D}\,(2\Lambda)}\emphng^{\rho}_{S_{\hpsi}}(w) + O\paren*{\sqrt{\frac{\log(1/\beta)}{m}}+\frac{\Lambda r}{\rho}\left(\frac{1}{\sqrt{m}}+\frac{\sqrt{\log(\frac{m}{\beta})\log(\frac{1}{\beta})}\log^{\frac{3}{4}}(\frac{1}{\delta})}{\sqrt{\varepsilon m}}\right)}. \label{bound:finite-dim}
\end{align}
\thnote{Same comment as previous section, terms independent of $\varepsilon$ are missing.}
The essence of the proof is to show that with probability $\geq 1-\beta/2$ over the randomness in $\hpsi$ (i.e., over the choice of $\omega_1, \ldots, \omega_D$), we have
\begin{align}
    \min\limits_{w\in \B^{2D}\,(2\Lambda)}\emphng^{\rho}_{S_{\hpsi}}(w)&\leq \min\limits_{h\in \cH_\Lambda}\emphng^{\rho}_S(h) + \frac{2\Lambda r}{\rho\sqrt{m}}. \label{bound:min-error-fin-dim-to-hilbert}
\end{align}
Combining (\ref{bound:finite-dim}) and (\ref{bound:min-error-fin-dim-to-hilbert}) yields the desired bound. 

To prove the bound in (\ref{bound:min-error-fin-dim-to-hilbert}), we will use the following fact. 
\mnote{I think this follows directly the fact the optimization problem
  for $\rho = 1$ coincides with the standard SVM one with
  $C = 1/(2 \mu m)$ and the box constraint for the dual solution of
  SVM, $0 \leq \alpha_i \leq C$}\rnote{Yes, indeed that's exactly what
  we discussed in the last meeting. I just wanted to add a short
  explanation for completeness. I may just elaborate on what I already
  said below after the lemma. I think I should also call it a fact
  rather than a lemma}.

\begin{fact}
\label{lem:dual_variables_bound}
Let $\mu >0$. Let $\psi \colon \cX \to \re$ denote the feature map
associated with the kernel $K$. Let
$h_{\mu} = \argmin_{h\in
  \cH_\Lambda}\paren*{\emphng^{\rho}_S(h)+\mu\norm{h}^2_{\,\mathbb{H}}}$. Then,
$h_\mu = \sum_{i = 1}^m \alpha_i \psi(x_i)$ for some $\alpha_i$,
$i \in [m]$, that satisfy:
$0 \leq y_i \alpha_i \leq \frac{1}{2\,m\,\mu\,\rho}$.
\end{fact} 
\noindent This fact simply follows from the dual formulation of the
optimization problem for kernel support vector machines (see, e.g.,
\cite[Section~6.3]{mohri2018foundations}) \rnote{If needed, I can
  briefly elaborate on how the constraint on $\alpha$ is equivalent to
  the constraint on the dual variables in the standard SVM
  formulation}. \thnote{Yes, please. I couldn't fully follow the
  details.} The fact asserts that the minimizer $h_\mu$ of the
regularized empirical hinge loss can be expressed as a linear
combination of $(\psi(x_i): i\in [m])$ (such assertion also follows
from the representer theorem) where the coefficients of the linear
combination (the dual variables) $\alpha=(\alpha_1, \ldots, \alpha_m)$
are bounded; namely, $\norm{\alpha}_1\leq \frac{1}{2\mu \rho}$.

Below, we set $\mu=\frac{r}{\Lambda \rho \sqrt{m}}$. Let
$\widehat{w}=\sum_{i=1}^m\alpha_i \hpsi(x_i)$ be a $2D$-dimensional
approximation of $h_\mu$. Observe that
\begin{align*}
\emphng^{\rho}_{S_{\hpsi}}(\widehat{w})-\emphng^{\rho}_{S}(h_\mu)
& = \frac{1}{m}\sum_{i=1}^m\left[\basichng^{\rho}(y_i\langle \widehat{w}, \hpsi(x_i)\rangle)-\basichng^{\rho}(y_i \langle h_\mu, \psi(x_i)\rangle_{\,\mathbb{H}}\right]\\
& \leq \frac{1}{\rho m}\sum_{i, j\in[m]}|\alpha_j|\lvert \langle \hpsi(x_i), \hpsi(x_j)\rangle - \langle \psi(x_i), \psi(x_j)\rangle_{\,\mathbb{H}}\rvert
\end{align*}
$\alpha_i$ where the inequality in the second line follows from the
fact that $\basichng^{\rho}$ is $\frac{1}{\rho}$-Lipschitz. Hence, by
Theorem~\ref{thm:RFF-approx}, with probability $\geq 1-\beta/2$ with
respect to the randomness in $\hpsi$, we have
\begin{align*}
\emphng^{\rho}_{S_{\hpsi}}(\widehat{w})-\emphng^{\rho}_{S}(h_\mu)&\leq \frac{2r^2}{\rho}\sqrt{\frac{\log(2m/\beta)}{D}}\norm{\alpha}_1\\
&\leq \frac{r^2}{\rho^2\mu}\sqrt{\frac{\log(2m/\beta)}{D}}\\
& \leq \frac{\Lambda r}{\rho \sqrt{m}},
\end{align*}
where the second inequality follows from the fact that
$\norm{\alpha}_1\leq \frac{1}{2\mu \rho}$, which follows from
Fact~\ref{lem:dual_variables_bound}, and the third inequality follows
from the setting of $D$ in step~\ref{step:set_D} in
Algorithm~\ref{alg:prvkermrg} and the setting of
$\mu = \frac{r}{\Lambda \rho \sqrt{m}}$. Moreover, we note that
$\widehat{w} \in \B^{2D}\,(2\Lambda)$. Indeed, conditioned on the same
event above (the kernel matrix is well approximated via $\hpsi$),
observe that
\begin{align*}
    \norm{\widehat{w}}_2^2 &= \sum_{i, j}\alpha_i\alpha_j \langle \hpsi(x_i), \hpsi(x_j)\rangle \\
    &\leq \sum_{i, j}\alpha_i\alpha_j \langle \psi(x_i), \psi(x_j)\rangle_{\,\mathbb{H}}+2r^2\sqrt{\frac{\log(2m/\beta)}{D}}\norm{\alpha}_1^2\\
    &\leq \norm{h}^2_{\,\mathbb{H}}+\frac{\Lambda^2}{2}\leq \frac{3}{2}\Lambda^2.
\end{align*}
Thus, we have $\norm{\widehat{w}}_2 < 2\Lambda$. Hence, we can assert
that with probability $\geq 1-\beta/2$ over the randomness in $\hpsi$, we have
\begin{align*}
    \min\limits_{w\in \B^{2D}\,(2\Lambda)}\emphng^{\rho}_{S_{\hpsi}}(w)&\leq \emphng^{\rho}_{S_{\hpsi}}(\widehat{w})\leq \emphng(h_\mu; S) + \frac{\Lambda r}{\rho \sqrt{m}}.
\end{align*}
\thnote{Should $\emphng(h_\mu; S)$ be $\emphng^\rho(h_\mu; S)$ in the
  above equation?}  Finally, note that
$\emphng^{\rho}_{S}(h_\mu)\leq
\min\limits_{h\in\cH_\Lambda}\emphng^{\rho}_S(h) + \mu\,\Lambda^2=
\min\limits_{h\in\cH_\Lambda}\emphng^{\rho}_S(h) + \frac{\Lambda
  r}{\rho \sqrt{m}}$. Hence, we arrive at the claimed bound
(\ref{bound:min-error-fin-dim-to-hilbert}), and thus, the proof is
complete.
\end{proof}

\section{DP Algorithms for
Learning Neural Networks with Margin Guarantees}
\ignore{
\section{Margin guarantees for neural networks}
}
\label{app:NNs}

\MrgNNs*
\ignore{ 
\begin{restatable}{thm}{MrgNNs}
Let $S\sim \sD^m$ for some distribution $\sD$ over $\B^d(r)\times \{\pm 1\}$. Let $\varepsilon>0, \beta \in (0, 1),$ and $\rho > 0$. There is an $\varepsilon$-DP algorithm that, given $S$ and margin parameter $\rho$, outputs an $L$-layer network $h^\prv$ with $N$ neurons per layer such that with probability at least $1-\beta$, we have
\begin{align*}
    R_{\sD}(h^\prv)&\leq \min\limits_{h\in \sH_{{\sf NN}^{\Lambda}}}\h R_S^{\rho}(h) + O\paren*{\frac{r (\eta \Lambda)^L\,\sqrt{N \log(Lm/\beta)\log(r\Lambda^L/\rho)}}{\rho \sqrt{m}}+\frac{r^2 (\eta\Lambda)^{2L} \, N\log(Lm/\beta)\log(r\Lambda^L/\rho)}{\rho^2 \varepsilon m}},
\end{align*}
where $\h R_S^{\rho}$ is the empirical $\rho$-margin loss and $\sH_{{\sf NN}^{\Lambda}}$ is the family of $L$-layer networks described above (with $N$ neurons per layer, sigmoid activation with parameter $\eta$, and weight matrices $W_1, \ldots, W_L$ with Frobenius norm bound $\Lambda$ on each of them).
\end{restatable}
}
\begin{proof}
First, note that our construction is indeed
$\varepsilon$-DP by the properties of the exponential mechanism. Thus,
we now turn to the proof of the margin bound. Our proof relies
on the following properties of the JL-transform.

\begin{lem}[Follows from Theorem~109 in \cite{nelson2010johnson}]
\label{lem:JL_NN}
Let $p, N, m, k \in \mathbb{N}$. Let $W\in \Rset^{p\times N}$. Let
$z_1, \ldots, z_m\in \Rset^p$. Let $\Phi$ be a random $k\times p$
matrix with entries drawn i.i.d. uniformly from $\{\pm
\frac{1}{\sqrt{k}}\}$. Let $\beta \in (0, 1)$. There is a constant $c>
0$ such that the following inequalities hold simultaneously with
probability at least $1 - \beta$:
\begin{align*}
    &\norm{\Phi W}^2_F \leq \norm{W}_F^2\paren*{1+c\,\sqrt{\frac{\log(m/\beta)}{k}}}\,,\\
    \forall i\in [m]:~& \norm{W^\top\Phi^\top\Phi z_i - W^\top z_i}_2 \leq c \norm{W}_F\norm{z_i}_2\,\sqrt{\frac{\log(m/\beta)}{k}} 
\end{align*}
\end{lem}

\noindent Consider the algorithmic construction described earlier. Let
$h_\ast\in \argmin\limits_{h\in \sH_{\sf NN}}\h R_S^\rho(h)$. Let
$W^\ast_1, \ldots, W^\ast_L$ denote the weight matrices of
$h_\ast$. Let $h^{\Phi}_\ast\in \sH^{\Phi}_{\sf NN}$ be the network
specified by the matrices $\wt W_1\triangleq\Phi_0 W^\ast_1, \ldots,
\wt W_L\triangleq\Phi_{L-1}W^\ast_L$. That is, the weight matrices of
$h^{\Phi}_\ast$ are given by $\Phi_0^\top\,\wt W_1=\Phi_0^\top\Phi_0
W^\ast_1, \ldots,$ $\Phi_{L-1}^\top\,\wt W_L=\Phi_{L-1}^\top
\Phi_{L-1}W^\ast_L$.



We make the following four claims. Combining those claims together
with the union bound immediately yields the margin bound of the
theorem. We first state those claims and then prove them.

\begin{claim}
  \label{claim:NN1}
There is a setting $k=O\paren*{\frac{r^2(2\eta
    \Lambda)^{2L}\log(Lm/\beta)}{\rho^2}}$ such that with probability
$1-\beta/4$ over the choice of $\Phi_0, \ldots, \Phi_{L-1}$, we have
$h^{\Phi}_\ast\in \sH^{\Phi}_{{\sf NN}^{2\Lambda}}$ and for all $i\in
[m]$
\begin{align*}
    \lvert h_\ast(x_i)- h^{\Phi}_\ast(x_i) \rvert &=  O\paren*{r (2\eta\Lambda)^L \sqrt{\frac{\log(Lm/\beta)}{k}}}.
\end{align*}
Consequently, with probability $1-\beta/4$,
\begin{align*}
    \h R^{0.5\rho}_S(h^{\Phi}_\ast)&\leq \h R^{\rho}_S(h_\ast).
\end{align*}
\end{claim}

\begin{claim}\label{claim:NN2}
Let $\h{h}^{\Phi}\in \argmin\limits_{h\in \h\sH^{\Phi}_{{\sf
      NN}^{2\Lambda}}}\h R_S(h)$. There exists a setting
$k=O\paren*{\frac{r^2(2\eta \Lambda)^{2L}\log(Lm/\beta)}{\rho^2}}$ for
the embedding parameter such that with probability $1-\beta/4$
\begin{align*}
    \h R_S(\h{h}^{\Phi})&\leq \h R^{0.5\rho}_S(h^{\Phi}_\ast).
\end{align*}
\end{claim}

\begin{claim}
  \label{claim:NN3}
Let $\h{h}^{\Phi}\in \argmin\limits_{h\in \h\sH^{\Phi}_{{\sf
      NN}^{2\Lambda}}}\h R_S(h)$. Let $k=O\paren*{\frac{r^2(2\eta
    \Lambda)^{2L}\log(Lm/\beta)}{\rho^2}}$. With probability
$1-\beta/4$ over the randomness of the exponential mechanism, we have
\begin{align*}
   \h R_S(h^\prv) &\leq \h R_S(\h{h}^{\Phi}) + O\paren*{\frac{r^2(2\eta \Lambda)^{2L} N \log(Lm/\beta)\log(r(4\eta\Lambda)^L/\rho)}{\rho^2 \varepsilon m}}.
\end{align*} 
\end{claim}

\begin{claim}
  \label{claim:NN4}
Let $k=O\paren*{\frac{r^2(2\eta
    \Lambda)^{2L}\log(Lm/\beta)}{\rho^2}}$. With probability
$1-\beta/4$ over the choice of $S\sim \sD^m$, we have
\begin{align*}
    R_\sD(h^\prv) &\leq \h R_S(h^\prv) + O\paren*{\frac{r(2\eta \Lambda)^{L} \sqrt{N \log(Lm/\beta)\log(r(\eta\Lambda)^L/\rho)}}{\rho \sqrt{m}}}.
\end{align*} 
\end{claim}
Recall that $\h \sH^{\Phi}_{{\sf NN}^{2\Lambda}}$ is a finite approximation of $\sH^{\Phi}_{{\sf NN}^{2\Lambda}}$ constructed via a $\gamma$-cover $\cC$ for $\B^{k\times N}(2\Lambda)\times \ldots\times\B^k(2\Lambda)$, where we choose $\gamma= \frac{\rho}{10 r(4\eta \Lambda)^{L-1}}$.
In particular, for any $W=(W_1, \ldots, W_L), W'=(W'_1, \ldots, W'_L) \in \cC$, we have $\norm{W-W'}_F=\sqrt{\sum_{j=1}^L\norm{W_j-W'_j}_F^2}\leq \gamma$. 
Given that $\cC$ is a $\gamma$-cover, we have $\lvert \h \sH^{\Phi}_{{\sf NN}^{2\Lambda}}\rvert$ 
$ = \lvert \cC\rvert = O\paren*{
\left(
\frac{\sqrt{L}\Lambda}{\gamma}
\right)^{k\times N}
}$. 
Namely, $\log(\lvert \h \sH^{\Phi}_{{\sf
    NN}^{2\Lambda}}\rvert)=O\paren*{k N
  \log(\frac{r(\eta\Lambda)^L}{\rho})}$. Given this, together with
the setting of $k$ in Claim~\ref{claim:NN4} and the fact that $h^\prv$ is in $\h \sH^{\Phi}_{{\sf NN}^{2\Lambda}}$, note that Claim~\ref{claim:NN4}
follows from a straightforward uniform convergence bound for the
hypotheses in $\h \sH^{\Phi}_{{\sf NN}^{2\Lambda}}$. Note also that
the proof of Claim~\ref{claim:NN3} follows directly from the standard
accuracy guarantee of the exponential mechanism when instantiated on
$\h \sH^{\Phi}_{{\sf NN}^{2\Lambda}}$. In particular, since the score
function is $-\h R_S(\cdot)$, with probability at least $1-\beta/4$,
the excess empirical loss of $h^\prv$ is bounded by
$O\paren*{\frac{\log(\lvert \cC\rvert/\beta)}{\varepsilon m}}$, which
yields the bound claimed in Claim~\ref{claim:NN3} given the bound on
$\lvert \cC\rvert$ above and the setting of $k$.

We now turn to the proofs of Claims~\ref{claim:NN1} and \ref{claim:NN2}. We start with the proof of Claim~\ref{claim:NN1}. 

For each $i\in [m]$ and each $j\in [L]$, let $v_{i,j}\in \Rset^N$
denote the output of the $j$-th layer of $h_\ast$ on input $x_i$ prior
to activation (i.e., $v_{i,j}$ is the input to the neurons of layer
$j+1$ when the input to the network $h_\ast$ is the $i$-th feature
vector $x_i$ in the dataset $S$). Analogously, for each $i\in [m]$ and
each $j\in [L]$, let $v^{\Phi}_{i,j}$ denote the output of the $j$-th
layer of $h^{\Phi}_\ast$ on input $x_i$ prior to activation. Also, let
$u_{i, j}\triangleq \psi(v_{i, j})-\psi(v^{\Phi}_{i, j})$, $i\in [m],
j\in [L]$.

As a direct corollary of Lemma~\ref{lem:JL_NN}, by applying the union
bound over the choice of $\Phi_0, \ldots, \Phi_{L-1}$, there is a
constant $\hat{c} > 0$ such that with probability $1 - \beta/4$ over the
choice of $\Phi_0, \ldots, \Phi_{L-1}$, for all $i\in [m], j\in [L]$,
we have
\begin{align}
    \norm{\Phi_{j-1} W^\ast_j}^2_F\leq \Lambda^2 & \paren*{1+\hat{c}\sqrt{\frac{\log(Lm/\beta)}{k}}}\,, \label{ineq:JL-NN-1}\\
    \norm{(W_j^{\ast})^{\top}\Phi_{j-1}^\top\Phi_{j-1} \psi(v_{i, j-1}) - (W_j^{\ast})^{\top} \psi(v_{i,j-1})}_2 &\leq \hat{c} \Lambda\norm{\psi(v_{i, j-1})}_2\,\sqrt{\frac{\log(Lm/\beta)}{k}},~~ j\neq 1,\label{ineq:JL-NN-2}\\
    \norm{(W_j^{\ast})^{\top}\Phi_{j-1}^\top\Phi_{j-1} u_{i, j-1} - (W_j^{\ast})^{\top} u_{i, j-1}}_2 &\leq \hat{c} \Lambda\norm{u_{i, j-1}}_2\,\sqrt{\frac{\log(Lm/\beta)}{k}},~~ j\neq 1,\label{ineq:JL-NN-3}\\
    \norm{(W_1^{\ast})^{\top}\Phi_{0}^\top\Phi_{0} x_i - (W_1^{\ast})^{\top} x_i}_2 &\leq \hat{c} \Lambda r\,\sqrt{\frac{\log(Lm/\beta)}{k}}\label{ineq:JL-NN-4}
\end{align}

We now condition on the event where all the above inequalities are
satisfied for the remainder of the proof. Below, we let
$\tau=\hat{c}\sqrt{\frac{\log(Lm/\beta)}{k}}<1$. First, from
(\ref{ineq:JL-NN-1}), there is a setting $k$ as indicated in the
statement of the claim, where $\norm{\Phi_{j-1} W^\ast_j}_F<
2\Lambda$. Thus, $h_\ast^{\Phi}\in \sH^{\Phi}_{{\sf NN}^{2\Lambda}}$.

Now, fix any $i\in [m]$. Define $\Gamma_j\triangleq \norm{(W^\ast_j)^\top\psi(v_{i, j-1})-(W^\ast_j)^\top\Phi_{j-1}^\top\Phi_{j-1}\psi(v^{\Phi}_{i, j-1})}_2$ for $j\in [L]$. Observe that 
\begin{align*}
    &\lvert h_\ast(x_i)- h^{\Phi}_\ast(x_i) \rvert = \Gamma_L\\
    =& \lvert (W^\ast_L)^\top\psi(v_{i, L-1})-(W^\ast_L)^\top\Phi_{L-1}^\top\Phi_{L-1}\psi(v^{\Phi}_{i, L-1}) \rvert \\
    \leq &\lvert (W^\ast_L)^\top\psi(v_{i, L-1})-(W^\ast_L)^\top\Phi_{L-1}^\top\Phi_{L-1}\psi(v_{i, L-1})\rvert  \\
    & +\lvert (W^\ast_L)^\top\Phi_{L-1}^\top\Phi_{L-1}\psi(v_{i, L-1}) - (W^\ast_L)^\top\Phi_{L-1}^\top\Phi_{L-1}\psi(v^{\Phi}_{i, L-1}) \rvert \\
    \leq & \tau \Lambda \norm{\psi(v_{i, L-1})}_2 + \lvert(W^\ast_L)^\top\Phi_{L-1}^\top\Phi_{L-1}\left(\psi(v_{i, L-1}) - \psi(v^{\Phi}_{i, L-1}) \right)\rvert \quad \big(\text{follows from (\ref{ineq:JL-NN-2}) and the fact }W_L^\ast \in \B^{N}(\Lambda)\big)\\
    = &\tau \Lambda \norm{\psi(v_{i, L-1})}_2 + \lvert(W^\ast_L)^\top\Phi_{L-1}^\top\Phi_{L-1}u_{i, L-1}\rvert \qquad\qquad\qquad\qquad~~ \big(\text{by definition of }u_{i, L-1} \text{ given above}\big) \\
    \leq & \tau \Lambda \norm{\psi(v_{i, L-1})}_2 + (1+\tau)\Lambda\norm{u_{i, L-1}}_2 \qquad\qquad\qquad\qquad\qquad\quad~  \big(\text{follows from (\ref{ineq:JL-NN-3})}\big) \\
    \leq & \tau \Lambda \norm{\psi(v_{i, L-1})}_2 + 2\Lambda\norm{\psi(v_{i, L-1}) - \psi(v^{\Phi}_{i, L-1})}_2\\
    \leq & \tau \eta\Lambda \norm{v_{i, L-1}}_2+ 2\eta\Lambda\norm{v_{i, L-1} - v^{\Phi}_{i, L-1}}_2 \qquad\qquad\qquad\qquad\qquad~~~~\big(\text{since  }\psi \text{ is }\eta\text{-Lipschitz and }\psi(0)=0\big)\\
    = & \tau \eta\Lambda \norm{v_{i, L-1}}_2+ 2\eta\Lambda\norm{(W^\ast_{L-1})^\top\psi(v_{i, L-2}) - (W^\ast_{L-1})^\top\Phi_{L-2}^\top\Phi_{L-2}\psi(v^{\Phi}_{i, L-2})}_2\\
    = &  \tau \eta\Lambda \norm{v_{i, L-1}}_2+ 2\eta \Lambda\Gamma_{L-1}
\end{align*}
Hence, we obtain $\Gamma_L \leq \tau \eta\Lambda \norm{v_{i,
    L-1}}_2+ 2\eta \Lambda\Gamma_{L-1}.$ Before we solve this
recurrence, we first unravel the term $\norm{v_{i, L-1}}_2$. Note that
\begin{align*}
    \norm{v_{i, L-1}}_2&= \norm{(W^\ast_{L-1})^\top\psi(v_{i, L-2})}_2\\
    &\leq \norm{W^\ast_{L-1}}_F \cdot \norm{\psi(v_{i, L-2})}_2\\
    &\leq \eta \Lambda \norm{v_{i, L-2}}_2
\end{align*}
Proceeding recursively, we obtain
\begin{align*}
    \norm{v_{i, L-1}}_2&\leq \eta^{L-2}\Lambda^{L-1}\norm{x_i}_2 \leq r \eta^{L-2}\Lambda^{L-1}.
\end{align*}
Plugging this in the recurrence for $\Gamma_L$ above yields
\begin{align*}
    \Gamma_L \leq \tau r \eta^{L-1}\Lambda^L+ 2\eta \Lambda\Gamma_{L-1}.
\end{align*}
Unraveling this recursion (and using (\ref{ineq:JL-NN-4}) in the last step of the recursion) yields
\begin{align*}
\lvert h_\ast(x_i)- h^{\Phi}_\ast(x_i)\rvert &= \Gamma_L \leq r (2\eta\Lambda)^{L-1}\Lambda\, \tau= \hat{c}\sqrt{\frac{\log(Lm/\beta)}{k}} r(2\eta\Lambda)^{L-1}\Lambda.
\end{align*}
Note that choosing $k= \frac{10 \hat{c}^2 r^2
  (2\eta\Lambda)^{2(L-1)}\Lambda^2\log(Lm/\beta)}{\rho^2}$ guarantees
$\lvert h_\ast(x_i)- h^{\Phi}_\ast(x_i)\rvert < \frac{\rho}{2}$ for
all $i \in [m]$. Hence, as in the argument of the proof of
Theorem~\ref{thm:prv_mrg_class}, this implies that for all $i\in [m],$
$y_i h_\ast(x_i) > \rho \Rightarrow y_i h^{\Phi}_\ast(x_i)>
\frac{\rho}{2}$. Thus, $\h R^{0.5\rho}_S(h^{\Phi}_\ast)\leq \h
R^{\rho}_S(h_\ast)$. This concludes the proof of
Claim~\ref{claim:NN1}.

Finally, we prove Claim~\ref{claim:NN2}. 

As shown in Claim~\ref{claim:NN1}, we have $h^{\Phi}_\ast\in
\sH^{\Phi}_{{\sf NN}^{2\Lambda}}$. Since $\h \sH^{\Phi}_{{\sf
    NN}^{2\Lambda}}$ is a $\gamma$-cover of $\sH^{\Phi}_{{\sf
    NN}^{2\Lambda}}$, there exists $\tilde h \in \h\sH^{\Phi}_{{\sf
    NN}^{2\Lambda}}$ that ``approximates'' $h^{\Phi}_\ast$. Namely,
there is $\tilde h \in \h\sH^{\Phi}_{{\sf NN}^{2\Lambda}}$ defined by
matrices $(\wt W_1, \ldots, \wt W_L)\in \cC$ such that
$$\sum_{j=1}^L\norm{\wt W_j - \Phi_{j-1}W^\ast_j}_F^2\leq \gamma^2,$$
where, as defined before, $(\Phi_0 W^\ast_1, \ldots,
\Phi_{L-1}W^{\ast}_L)$ are the matrices defining $h^{\Phi}_\ast$. We
choose $\gamma = \frac{\rho}{10 r (4\eta \Lambda)^{L-1}}$.

To simplify notation, we will denote
$$W^{\Phi, \ast}_j\triangleq \Phi_{j-1}W^\ast_j, ~\forall ~j\in [L].$$
As before, for each $i\in [m], j\in [L]$, we let $v^{\Phi}_{i,j}\in
\Rset^N$ denote the output of the $j$-th layer of $h^{\Phi}_\ast$ on
input $x_i$ prior to activation, and let $\tilde v_{i, j}$ denote the
output of the $j$-th layer of $\tilde h$ on input $x_i$ prior to
activation.

Again, as a corollary of Lemma~\ref{lem:JL_NN} (by applying the union
bound over the choice of $\Phi_0, \ldots, \Phi_{L-1}$), there is a
constant $\hat{c}>0$ such that with probability $1-\beta/4$ over the
choice of $\Phi_0, \ldots, \Phi_{L-1}$, for all $i\in [m], j\in [L]$,
we have
\begin{align}
    \norm{\Phi_{j-1} \psi(\tilde v_{i, j-1})}^2_2\leq &
    \norm{\psi(\tilde v_{i,
        j-1})}^2_2\paren*{1+\hat{c}\sqrt{\frac{\log(Lm/\beta)}{k}}},\,
    ~j\neq 1 \label{ineq:JL-NN-I}\\ \norm{\Phi_{j-1} \left(\psi(\tilde
      v_{i, j-1})-\psi(v^{\Phi}_{i, j-1})\right)}^2_2\leq &
    \norm{\psi(\tilde v_{i, j-1})-\psi(v^{\Phi}_{i,
        j-1})}^2_2\paren*{1+\hat{c}\sqrt{\frac{\log(Lm/\beta)}{k}}},\,~j\neq
    1\label{ineq:JL-NN-II}\\ \norm{\Phi_{0} x_i}^2_2\leq &
    r^2\paren*{1+\hat{c}\sqrt{\frac{\log(Lm/\beta)}{k}}} \label{ineq:JL-NN-III}
\end{align}
We will condition on the event above for the remainder of the
proof. Note that for the setting of $k$ as in Claim~\ref{claim:NN1},
we have $\paren*{1+\hat{c}\sqrt{\frac{\log(Lm/\beta)}{k}}}<2$.

For each $j \in [L],$ define
\[
\Delta_j \triangleq \norm{\wt W_j^\top \Phi_{j-1} \psi(\tilde v_{i,
    j-1})- W^{\Phi, \ast}_j\Phi_{j-1}\psi(v^{\Phi}_{i, j-1})}_2.
\]
Fix any $i\in [m]$. Observe 
\begin{align*}
    &\lvert \tilde h(x_i) - h^{\Phi}_\ast(x_i)\rvert= \Delta_L\\
    \leq & \lvert \wt W_L^\top \Phi_{L-1}\psi(\tilde v_{i, L-1}) -(W_L^{\Phi, \ast})^\top\Phi_{L-1}\psi(\tilde v_{i, L-1})\rvert\\
    & + \lvert (W_L^{\Phi, \ast})^\top\Phi_{L-1}\psi(\tilde v_{i, L-1})-\wt (W_L^{\Phi, \ast})^\top \Phi_{L-1}\psi(v^{\Phi}_{i, L-1})\rvert\\
    \leq &\norm{\wt W_L - W_L^{\Phi, \ast}}_F  \norm{\Phi_{L-1} \psi(\tilde v_{i, L-1})}_2 + \norm{W_L^{\Phi, \ast}}_F \norm{\Phi_{L-1} \left(\psi(\tilde v_{i, L-1})-\psi(v^{\Phi}_{i, L-1})\right)}_2\\
    \leq & \gamma\, \norm{\Phi_{L-1} \psi(\tilde v_{i, L-1})}_2 + 2\Lambda\, \norm{\Phi_{L-1} \left(\psi(\tilde v_{i, L-1})-\psi(v^{\Phi}_{i, L-1})\right)}_2 \quad \big(\cC \text{ is }\gamma\text{-cover and } W_L^{\Phi, \ast}\in \B^{k\times N}(2\Lambda)\big)\\
    \leq & \sqrt{2}\,\gamma\, \norm{\psi(\tilde v_{i, L-1})}_2+2\sqrt{2}\Lambda\, \norm{\psi(\tilde v_{i, L-1})-\psi(v^{\Phi}_{i, L-1})}_2 \qquad~~~~~ \big(\text{follows from (\ref{ineq:JL-NN-I})-(\ref{ineq:JL-NN-II})}\big)\\
    \leq & \sqrt{2}\,\gamma \eta\, \norm{\tilde v_{i, L-1}}_2+2\sqrt{2}\,\eta\Lambda\, \norm{\tilde v_{i, L-1}-v^{\Phi}_{i, L-1}}_2 \qquad\qquad\qquad~~~~ \big(\psi \text{ is }\eta\text{-Lipschitz and } \psi(0)=0\big)\\
    \leq & \sqrt{2}\,\gamma \eta\, \norm{\wt W_{L-1}^\top \Phi_{L-2}\psi(\tilde v_{i, L-2})}_2+2\sqrt{2}\,\eta\Lambda\,\norm{\wt W_{L-1}^\top \Phi_{L-2} \psi(\tilde v_{i, L-2})- W^{\Phi, \ast}_{L-1}\Phi_{L-2}\psi(v^{\Phi}_{i, L-2})}_2\\
    = & \sqrt{2}\,\gamma \eta\, \norm{\wt W_{L-1}^\top \Phi_{L-2}\psi(\tilde v_{i, L-2})}_2+2\sqrt{2}\,\eta\Lambda\, \Delta_{L-1}
\end{align*}
Hence, we arrive at a recursive bound
$$ \Delta_L \leq \sqrt{2}\,\gamma \eta\, \norm{\wt W_{L-1}^\top \Phi_{L-2}\psi(\tilde v_{i, L-2})}_2+2\sqrt{2}\,\eta\Lambda\, \Delta_{L-1}.$$
Before proceeding, we first unravel the term $\norm{\wt W_{L-1}^\top \Phi_{L-2}\psi(\tilde v_{i, L-2})}_2$. Let's denote this term as $B_{L-1}$. Observe that
\begin{align*}
    B_{L-1}=&\norm{\wt W_{L-1}^\top \Phi_{L-2}\psi(\tilde v_{i, L-2})}_2\\
    \leq & \norm{\wt W_{L-1}}_F\norm{\Phi_{L-2}\psi(\tilde v_{i, L-2})}_2\\
    \leq & 2\sqrt{2}\Lambda \norm{\psi(\tilde v_{i, L-2})}_2\\
    \leq & 2\sqrt{2} \eta \Lambda \norm{\tilde v_{i, L-2}}_2\\
    = & 2\sqrt{2} \eta \Lambda \norm{\wt W^\top_{L-2}\Phi_{L-3}\psi(\tilde v_{i, L-3})}_2\\
    = & 2\sqrt{2} \eta \Lambda B_{L-2}
\end{align*}
Thus, continuing recursively, we get $B_{L-1}\leq r\eta^{L-2}(2\sqrt{2} \Lambda)^{L-1}$ (where in the last step of the recursion, we use (\ref{ineq:JL-NN-III})). Plugging this back in the recursive bound for $\Delta_L$, we get 
\begin{align*}
    \Delta_L\leq & \sqrt{2} \gamma r (2\sqrt{2}\eta \Lambda)^{L-1}+2\sqrt{2}\,\eta\,\Lambda\,\Delta_{L-1}
\end{align*}
Unraveling this recurrence yields
\begin{align*}
    \Delta_L\leq & \sqrt{2} \gamma r L (2\sqrt{2}\eta \Lambda)^{L-1}\\
    \leq& 2\sqrt{2}\gamma r (4\eta \Lambda)^{L-1}
\end{align*}
Thus, by the choice of $\gamma$, we have $\lvert \tilde h(x_i) - h^{\Phi}_\ast(x_i)\rvert < \frac{\rho}{2}$ for all $i\in [m]$. Hence, as before, we have $\paren*{y_i h^{\Phi}_\ast(x_i)>\rho/2} \Rightarrow
\paren*{y_i\tilde h(x_i)>0}$ for all $i\in [m]$, which implies that 
$$\h R_S(\tilde h)\leq \h R^{0.5 \rho}_S(h^{\Phi}_\ast).$$
Since $\h{h}^{\Phi}\in \argmin\limits_{h\in \h\sH^{\Phi}_{{\sf NN}^{2\Lambda}}}\h R_S(h)$, then we have $\h R_S(\h{h}^{\Phi})\leq R_S(\tilde h)$. Therefore,  we can write
\[
\h R_S(\h{h}^{\Phi})\leq \h R^{0.5 \rho}_S(h^{\Phi}_\ast).
\]
This concludes the proof of Claim~\ref{claim:NN2} and completes the proof of Theorem~\ref{thm:MrgNNs}.
\end{proof}

\section{Label DP Algorithms with Margin Guarantees}
\ignore{
\section{Margin guarantees for label differential privacy}
}
\label{app:label}

\MargLabel*
\begin{proof}
  The $\varepsilon$-differential privacy guarantee follows directly
  from the properties of the exponential mechanism. In particular,
  given the finite class $\hH$ and the score function
  $-\widehat R^{\rho/2}_S(h),~h\in\hH$, the algorithm becomes an
  instantiation of the exponential mechanism
  \cite{mcsherry2007mechanism}.

  We focus on proving the utility guarantee in the rest of the
  proof. If $m < \frac{64 \fatshatteringbound \log(2/\beta)}{\varepsilon}$, then
  the bound follows trivially. Hence in the rest of the paper, we
  focus on the regime
  $m \geq \frac{64\fatshatteringbound \log(2/\beta)}{\varepsilon}$.  By
  definition of $\hH$, for any $h \in \sH$ there exists $g \in \hH$
  such that for any $x \in x^m_1$,
\[
|g(x) - h(x)| \leq \frac{\rho}{2}.
\]
Thus, for any $y \in \curl*{-1, +1}$ and $x \in x^m_1$, we have
$|yg(x) - yh(x)| \leq \rho/2 $, which implies:
\[
1_{y g(x) \leq \rho/2} \leq 1_{y h(x) \leq \rho}.
\]
Let $h^*_S\in \argmin\limits_{h\in\cH}\widehat R^\rho_S(h)$.  By the
construction of $\hH$ and the above argument,
\begin{equation}
    \label{eq:lab_proof_1}
\min_{h \in \hH} \widehat R^{\rho/2}_S(h) \leq \widehat R^{\rho}_S(h^*_S).
\end{equation}
We now bound the size $\hH$. 
\thnote{Mehryar's comment: add $\rho to \rho / r$.}
\[
|\hH| = \cN_\infty(\sH_\rho, \rho/2, x_1^m).
\]
By \cite[Proof of theorem 2]{Bartlett1998}, we have
\[
\log \max_{x_1^{m}} [\cN_\infty(\sH_\rho, \tfrac{\rho}{2}, x_1^{m})]
\leq  1 + d' \log_2(2c^2m) \log_2\frac{2cem}{d'},
\]
where
$d' = \fat_{\frac{\rho}{32}}(\sH_\rho) \leq
\fat_{\frac{\rho}{32}}(\sH) = d$ and $c=17$. Given the bound on the
sample size in the theorem statement and the properties of the
exponential mechanism \cite{mcsherry2007mechanism}, value of $m$,
with probability at least $ 1-\beta/2,$
\begin{align}
    \widehat R^{\rho/2}_S(\hprv)&\leq \min_{h \in \hH}\widehat R^\rho_S(h)+ \frac{32\fatshatteringbound \log(2/\beta)}{\varepsilon m} \nonumber \\
    & \leq \widehat R^{\rho}_S(h^*_S) + \frac{32\fatshatteringbound \log(2/\beta)}{\varepsilon m}. \label{eq:lab_proof_2}
\end{align}
By Lemma~\ref{lem:rel_margin_bound}, with probability at least
$1- \beta/2$,
\begin{align}
    R_\sD(\hprv)&\leq \widehat R^{\rho/2}_S(\hprv)+  2\sqrt{\widehat R^{\rho/2}_S(h) \frac{\fatshatteringbound}{m}}  
+ \frac{ \fatshatteringbound}{m}, \label{eq:lab_proof_3}
\end{align}
where
$\fatshatteringbound = 1 + d \log_2(2c^2m) \log_2\frac{2cem}{d} + \log
\frac{2}{\beta}$, $c=17$, and $d =
\fat_{\frac{\rho}{32}}(\sH)$. Combining~\eqref{eq:lab_proof_2}
and~\eqref{eq:lab_proof_3} yields
\[
R_\sD(\hprv)\leq \widehat R^{\rho}_S(h^*_S) + 2\sqrt{\widehat R^{\rho}_S(h^*_S) \frac{\fatshatteringbound}{m}}  
+ \frac{ 2\fatshatteringbound}{m} + \frac{64\fatshatteringbound \log(2/\beta)}{\varepsilon m}.
\]
The lemmas follows by observing that if
$h^* \in \argmin\limits_{h\in\cH}\widehat R^\rho_S(h)$, if and only if
$h^* \in \argmin\limits_{h\in\cH}\widehat R^\rho_S(h) + 2 \sqrt{\widehat
  R^{\rho}_S(h) \frac{\fatshatteringbound}{m}}$.
\end{proof}


\ignore{
\section{Optimal confidence-margin selection}
\label{app:rho}
}

\section{Confidence Margin Parameter Selection}
\label{app:rho}

\rnote{Make the algorithm/result in this section more general (works
  for any of the previous algorithms in Sections~4.1, 4.2, and 5)}
  
The algorithms of Section~\ref{sec:linear}, Section~\ref{sec:kernels} and Section~\ref{sec:label}
can all be augmented to include the selection
of the confidence margin parameter $\rho$ by using an exponential mechanism. All of the proposed algorithms in the previous sections output $w^\prv$ such that
\[
R_{\sD}(w^\prv) \leq F(\rho, \genericloss_\rho(S)),
\]
where $\genericloss_\rho(S)$ is either the minimum $\rho$-margin loss or the minimum $\rho$-hinge loss. Furthermore, in all our results, for any fixed $t$, $F(\rho, t)$ is a non-increasing function of $\rho$ and $\genericloss_\rho(S)$ is a non-decreasing function of $\rho$ for any $S$. 
Suppose we have an algorithm $\cA$ such that the above inequality holds. We can then augment it with an exponential mechanism algorithm to select a near-optimal margin $\rho$. Let $\hmax$ be an upper bound on $\max_{x \in \cX} \max_{h \in \sH} |h(x)|$. For example, $\hmax = \Lambda r$ for linear classifiers. If $\hmax > \rho$, then the bound $F(\rho; \genericloss_\rho(S))$ becomes trivial (i.e.,
$\Omega(1)$). Similarly, the bound typically becomes trivial  when $\rho \lesssim \frac{\hmax}{\sqrt{m}}$. It is easy to see this property for linear classifiers, for other models such as neural networks with label privacy it can be obtained by bounds on fat-shattering dimension \cite{bartlett1999generalization}. Hence, without
loss of optimality, we will seek an approximation for $\rho_\opt$ that
minimizes $F(\rho; \genericloss_\rho(S))$ for $\rho \in \left[\frac{\hmax}{\sqrt{m}}, \, \hmax\right]$. To do this,
we define a finite grid over the above interval: $\cV\triangleq
\left\{\rho_j\triangleq 2^{-j}\,\hmax~ j \in [J]\right\},$ where
$J=\frac{1}{2} \log(m)$.  We use an instantiation of the generalized exponential
mechanism, with score function $-F(\rho; \genericloss_{\rho}(S)),~\rho\in\cV$ and
privacy parameter $\varepsilon$, to select $\rho^\ast\in \cV$ that
approximately minimizes $F(\rho; \genericloss_\rho(S))$ over $\rho \in \cV$. We use the generalized exponential mechanism as the  sensitivity of $-F(\rho_j;  \genericloss_{\rho_j}(S))$ depends on $\rho_j$. We then run
$\cA$ with margin parameter $\rho=\rho^\ast$ to
output the final parameter vector $w^\prv$. For clarity, we include a
formal description of the full algorithm in
Algorithm~\ref{alg:full-prvmrg}. We now state the guarantee of the augmented algorithms.

\begin{algorithm}[t]
	\caption{$\cA_{\prv\mrgn}$: Algorithm to select confidence margin}
	\begin{algorithmic}[1]
		\REQUIRE Dataset $S=\left((x_1, y_1), \ldots, (x_m, y_m)\right) \in \left(\cX\times \{\pm 1\}\right)^m$; algorithm $\cA$;  bound $F(\rho, \genericloss_\rho(S))$; $\hmax$ an upper bound on $\rho$; privacy parameters $\varepsilon >0, \delta \geq 0$; and confidence parameter $\beta > 0$.
        \STATE Let $\cV\triangleq \left\{\rho_j\triangleq 2^{-j}\,\hmax:~ j \in [J]\right\},$ where $J=\frac{1}{2} \log(m)$.
        \STATE Run the generalized exponential mechanism \cite[Algorithm 1]{raskhodnikova2016lipschitz} over $\cV$ with privacy parameter $\varepsilon$ and score function $-F(\rho_j; \genericloss_{\rho_j}(S))$ for $\rho_j\in\cV$, to select $\rho^\ast \in \cV$.\label{stp:exp-mech-rho} 
        \STATE Run $\cA$ on the dataset $S$ with margin parameter $\rho^\ast$ and privacy parameters $(\varepsilon, \delta)$, confidence parameter $\beta$, and return its output $w^\prv$. \label{stp:alg-prvmrg}
	   	\end{algorithmic}
	\label{alg:full-prvmrg}
\end{algorithm}

\begin{restatable}{lem}{FullPrvMarg}
\label{lem:full-prvmrg}
Let $\beta \in (0,1)$. Suppose $S\sim \sD^m$ for some distribution $\sD$ over $\cX\times \cY$. 
Suppose $\cA$ is $(\varepsilon,\delta)$ differentially private and its output satisfies $R_{\sD}(w^\prv) \leq F(\rho, \genericloss_\rho(S))$ with probability at least $1-\beta$. Furthermore, for any $t$, let $F(\rho, t)$ be a non-increasing function of $\rho$ and  $\genericloss_\rho(S)$ is a non-decreasing function of $\rho$ for any $S$. Then, Algorithm~\ref{alg:full-prvmrg} is $(2\varepsilon,\delta)$-differentially private and with probability at least $1-2\beta$, the output $w^\prv$ satisfies: 
\[
R_{\sD}(w^\prv) \leq \min_{\rho \in \left[\frac{\hmax}{\sqrt{m}}, \hmax \right]} F(\rho/2, \genericloss_\rho(S)) + \frac{\Delta_{\rho/2}(F) }{\varepsilon}\cdot \log\left(\frac{\log(m)}{\beta}\right),
\]
where $\Delta_\rho(F)$ is a non-decreasing function of $\rho$ and is an upper bound on the sensitivity of $F$ given by
$
\Delta_\rho(F) = \max_{S, S': d(S, S') = 1} |F(\rho, \genericloss_\rho(S))- F(\rho, \genericloss_\rho(S'))|
$
and $d(S, S')$ is the number of samples in which $S$ and $S'$ differ.
\end{restatable}
\begin{proof}
  The privacy guarantee follows from the basic composition property of
  differential privacy together with the fact that the generalized exponential
  mechanism invoked in step~\ref{stp:exp-mech-rho} is
  $\varepsilon$-differentially private and
  $\cA$ is $(\varepsilon, \delta)$-differentially private.
  
 We now turn to the proof of the error bound. 
  Note that there exists $\hat{\rho}\in \cV$ such that
$\hat{\rho} \leq \rho_\opt < 2 \cdot \hat{\rho}$. By the properties of the
generalized exponential mechanism \cite[Theorem I.4]{raskhodnikova2016lipschitz} and the fact that
the sensitivity of $F(\rho, \genericloss_\rho(S))$ is $\Delta_\rho(F)$, with probability at least
$1-\beta$ we have
\begin{align}
    F(\rho^\ast; \genericloss_{\rho^\ast}(S))&\leq \min\limits_{\rho\in\cV}F(\rho; \genericloss_\rho(S))+\frac{\Delta_\rho(F) }{\varepsilon}\cdot \log\left(\frac{\log(m)}{\beta}\right)\nonumber \\
    &\leq F(\hat{\rho}; \genericloss_{\hat{\rho}}(S))+\frac{\Delta_{\hat{\rho}}(F) }{\varepsilon} \cdot \log\left(\frac{\log(m)}{\beta}\right)\nonumber \\
        &\leq F(\hat{\rho}; \genericloss_{\hat{\rho}}(S))+\frac{\Delta_{\rho_{\opt/2}}(F) }{\varepsilon} \cdot \log\left(\frac{\log(m)}{\beta}\right)\nonumber \\
      &\leq F(\hat{\rho}; \genericloss_{\rho_\opt}(S))+\frac{\Delta_{\rho_{\opt/2}}(F) }{\varepsilon}\cdot \log\left(\frac{\log(m)}{\beta}\right)\nonumber \\
    &\leq F(\rho_\opt/2; \genericloss_{\rho_\opt}(S))+\frac{\Delta_{\rho_{\opt/2}}(F) }{\varepsilon}\cdot \log\left(\frac{\log(m)}{\beta}\right)\label{eq:rho1},
\end{align}
 where the last two inequalities follow from the fact that for any $t$, let $F(\rho, t)$ be a non-increasing function of $\rho$ and  $\genericloss_\rho(S)$ is a non-decreasing function of $\rho$ for any $S$.
  By the assumption on $\cA$, with probability $1-\beta$, 
 \begin{equation}
 R_\sD(w^\prv) \leq F(\rho^\ast; \genericloss_{\rho^\ast}(S)).
 \label{eq:rho2}
 \end{equation}
 Combining~\eqref{eq:rho1} and~\eqref{eq:rho2} yields the lemma. The error probability follows by the union bound.
\end{proof}

The above lemma can be combined with any of the algorithms of Section~\ref{sec:linear}, Section~\ref{sec:kernels} and Section~\ref{sec:label}. We instantiate it for $\cA_{\eff\prv\mrg}$ in the following corollary. Below, we compute sensitivity for the bounds on other algorithms, which can be used to get similar guarantees. 
\begin{restatable}{cor}{CorFullPrvMarg}
\label{thm:full-prvmrg}
Let $\beta \in (0,1)$ and $m\in \mathbb{N}$. Suppose $S\sim \sD^m$ for some distribution $\sD$ over $\cX\times \cY$. Recall that by Theorem~\ref{thm:prv_mrg_hinge}, the output of Algorithm~\ref{alg:effprvmrg} (denoted by  $w'$) with probability at least $1-\beta/2$ satisfies, $R_{\sD}(w') \leq F(\rho, \genericloss_\rho(S))$, where $ \genericloss_\rho(S) = \min\limits_{w\in \B^d(\Lambda)}\emphng^{\rho}_S(w)$ and
\[
F(\rho', t) = t
    + 
    O\paren*{\sqrt{\frac{\log(1/\beta)}{m}}+\frac{\Lambda r}{\rho'}\left(\frac{1}{\sqrt{m}}+\frac{\sqrt{\log(\frac{m}{\beta})\log(\frac{1}{\beta})}\log^{\frac{3}{4}}(\frac{1}{\delta})}{\sqrt{\varepsilon m}}\right)}.
\]
Let $w^\prv$ be the output of Algorithm~\ref{alg:full-prvmrg} with inputs $S$, privacy parameters $\varepsilon/2, \delta$, bound $F(\rho, \genericloss_\rho(S))$ , algorithm $\cA_{\eff\prv\mrg}$, confidence parameter $\beta/2$, and $\hmax = \Lambda r$. 
Then $w^\prv$ is $(\varepsilon ,\delta)$ differentially private. Furthermore, with probability at least $1 - \beta$,
\[
R_{\sD}(w^\prv) \leq \min_{\rho \in \left[ \frac{\Lambda r}{\sqrt{m}}, \Lambda r \right]}F(\rho/2, \genericloss_\rho(S)) + O \left( \frac{\Lambda r }{m \rho \varepsilon}\cdot \log\left(\frac{\log(m)}{\beta}\right) \right).
\]
\end{restatable}

\begin{lem}
\label{lem:sensitivities}
Fix $\rho > 0$. Let the functions $F_1$, $F_2$, $F_3$, $F_4$ and $F_5$ are defined as follows:
\begin{align*}
F_1(\rho', \genericloss_\rho(S)) &=\min\limits_{w\in \B^d(\Lambda)}\widehat R^\rho_S(w)+ O \left( \sqrt{\widehat R^\rho_S(w) \left( \frac{\Lambda^2 r^2\log^2\paren*{\frac{m}{\beta}}}{m(\rho')^2} + \frac{\log\paren*{\frac{1}{\beta}}}{m} \right)} + \Gamma \right), \\
F_2(\rho', \genericloss_\rho(S)) &= \min\limits_{w\in \B^d(\Lambda)}\emphng^{\rho}_S(w)
    + 
    O\paren*{\sqrt{\frac{\log(1/\beta)}{m}}+\frac{\Lambda r}{\rho'}\left(\frac{1}{\sqrt{m}}+\frac{\sqrt{\log(\frac{m}{\beta})\log(\frac{1}{\beta})}\log^{\frac{3}{4}}(\frac{1}{\delta})}{\sqrt{\varepsilon m}}\right)},
\\
F_3(\rho', \genericloss_\rho(S)) &= \min\limits_{h\in \cH_\Lambda}\emphng^{\rho}_S(h) +   O\paren*{\sqrt{\frac{\log(1/\beta)}{m}}+\frac{\Lambda r}{\rho'}\left(\frac{1}{\sqrt{m}}+\frac{\sqrt{\log(\frac{m}{\beta})\log(\frac{1}{\beta})}\log^{\frac{3}{4}}(\frac{1}{\delta})}{\sqrt{\varepsilon m}}\right)},
\\
F_4(\rho', \genericloss_\rho(S) )&=   \min_{h \in \sH}\widehat R^\rho_S(h) +  2\sqrt{\min_{h \in \sH}\widehat R^\rho_S(h)}\sqrt{ \frac{\fatshatteringbound}{m}}  
+ \frac{ 2\fatshatteringbound}{m} + \frac{64\fatshatteringbound \log\paren*{\frac{2}{\beta}}}{\varepsilon m},
\\
F_5(\rho', \genericloss_\rho(S) ) &= \min\limits_{~~~ h\in \sH_{{\sf NN}^{\Lambda}}}\h R_S^{\rho}(h)
  + O\paren*{\frac{r (2\eta \Lambda)^L\,\sqrt{N \theta}}{\rho' \sqrt{m}}
    +\frac{r^2 (2\eta\Lambda)^{2L} \, N\theta}{(\rho')^2 \varepsilon m}},
\end{align*}
where $M$ is defined in Theorem~\ref{thm:label_dp} and $\Gamma$ is defined in Theorem~\ref{thm:prv_mrg_class}.
Then
\begin{align*}
\Delta_\rho(F_1) &= O \left( \frac{1}{m} + \frac{1}{m}\sqrt{  \frac{\Lambda^2 r^2\log^2\paren*{\frac{m}{\beta}}}{\rho^2} + \log\paren*{\frac{1}{\beta}}} \right).
\\
\Delta_\rho(F_2) &= O \left( \frac{\Lambda r}{m\rho} \right).
\\
\Delta_\rho(F_3) &= O \left( \frac{\Lambda r}{m\rho} \right).
\\
\Delta_\rho(F_4) &= O \left( \frac{1 + \sqrt{M}}{m}\right).
\\
\Delta_\rho(F_5) &= O \left(\frac{1}{m} \right).
\end{align*}
\end{lem}
\begin{proof}
We provide the proof for the bound on $\Delta_\rho(F_2)$. The proof for other quantities is similar and omitted. Let $S'$ and $S''$ be two samples that differ in at most one sample. Without loss of generality, let $F_1(\rho, \genericloss_\rho(S')) \geq F_2(\rho,  \genericloss_\rho(S'')) $. Let $$w' \in \argmin_{w \in \B^d(\Lambda)} \emphng^{\rho}_{S'}(w)
    + 
    O\paren*{\sqrt{\frac{\log(1/\beta)}{m}}+\frac{\Lambda r}{\rho}\left(\frac{1}{\sqrt{m}}+\frac{\sqrt{\log(\frac{m}{\beta})\log(\frac{1}{\beta})}\log^{\frac{3}{4}}(\frac{1}{\delta})}{\sqrt{\varepsilon m}}\right)}$$ and 
$$w'' \in  \argmin_{w \in \B^d(\Lambda)} \emphng^{\rho}_{S''}(w)
    + 
    O\paren*{\sqrt{\frac{\log(1/\beta)}{m}}+\frac{\Lambda r}{\rho}\left(\frac{1}{\sqrt{m}}+\frac{\sqrt{\log(\frac{m}{\beta})\log(\frac{1}{\beta})}\log^{\frac{3}{4}}(\frac{1}{\delta})}{\sqrt{\varepsilon m}}\right)}.$$ Then
\begin{align*}
F_1(\rho, \genericloss_\rho(S'))- F_2(\rho,  \genericloss_\rho(S''))
 & = \emphng^{\rho}_{S'}(w') - \emphng^{\rho}_{S''}(w'') \\
 & \leq \emphng^{\rho}_{S'}(w'') - \emphng^{\rho}_{S''}(w'') \\
 & \leq \frac{2}{m \rho} \max_{w \in \B^d(\Lambda), x \in \B^d(r)} |w \cdot x| \\
 & \leq \frac{2}{m \rho} \Lambda r.
\end{align*}
\end{proof}

\section{Example of high error for exponential mechanism}
\label{app:counter}

\begin{restatable}{lem}{CounterExample}
Let $d \geq c$ for some constant $c$ and $\rho \in [0, 1]$. There exists a distribution $\sD$ over $\B^d$ and a subset $\cH \in \cH_\lin$ such that the following hold:
\begin{itemize}
    \item \textbf{Realizable setting}: There exists a $h^\ast$ in
      $\cH$ such that $R_{\sD}(h^\ast) = 0$. \rnote{Why assume realizability? This then becomes exactly the same result as Nguyen et al.} \thnote{As discussed offline, this is a counter example, and hence its okay to make this assumption.} 

    \item \textbf{Only one good hypothesis}: For any $h$ in $\cH
      \setminus \{h^\ast\}$, $R_{\sD}(h) = 1$.

    \item \textbf{A good cover}: For any two $h_w, h_{w'} \in \cH$,
      their corresponding weights satisfy $|\langle w ,  w' \rangle| \geq 1/8$.

    \item \textbf{Exponential mechanism incurs high error}: Given $m <
      c' \cdot d / \varepsilon $ samples $\sD$, with probability at least
      $9/10$, the exponential mechanism on $-\widehat R^\rho_S(h)$ will select a
      $h$ such that $R_{\sD}(h) = 1$.
\end{itemize}
\end{restatable}

\begin{proof}
Let $\sD$ be defined as follows. Let $\sD(x)$ be a uniform distribution over $\{-1/\sqrt{d}, 1/\sqrt{d}\}^d$ and $y = 1$ if $x_1 > 0$, $0$ otherwise. The optimal hypothesis $h^\ast(x) = 1_{x_1 > 0}$ and satisfies $R_{\sD}(h^\ast) = 0$. Let $\cH = \{h^\ast\} \cup \{h_w : w \in \cW\}$, where $\cW$ is the largest set such that for all $w \in \cW$, $w_1 = -1/\sqrt{d}$ and for any two $w, w' \in \cW$, $|\langle w ,  w' \rangle| \geq 1/8$. By Gilbert-Varshamov bound, size of such a set is at least $2^{c \cdot d}$ for some constant $c$. Note that for any $\cH \setminus \{h^\ast\}$, $R_{\sD}(h) = 1$.

Now suppose we use exponential mechanism with score $-\widehat R^\rho_S(h)$. The probability of selecting the correct hypothesis is at most
\[
\frac{1}{\sum_{h \in \{h_w : w \in \cW\}} \exp(-\widehat R^\rho_S(h) \varepsilon /2m)} \leq \frac{1}{2^{c \cdot d} e^{-\varepsilon m/2 }} = e^{\varepsilon m/2 - c' d}. 
\]
Hence if $m < c'/d\varepsilon$, then the probability of choosing $h^\ast$ is at most $e^{-c' d/2} \leq 1/10$ for $d \geq \frac{2}{c'} + 3$.

\end{proof}

\newpage
\section{Example of a large norm hinge-loss minimizer}
\label{app:example}

Fix $\alpha \in [0, 1]$ and $\gamma \in (0, r)$.  Consider the
distribution $\sD$ on the real line (dimension one) defined as
follows: there is a probability mass of $\frac{\alpha}{2}$ at
coordinate $(+r, -1)$, a probability mass of $\frac{\alpha}{2}$ at
$(-r, +1)$, a probability mass of $\frac{1 - \alpha}{2}$ at $(+\gamma,
+1)$, and a probability mass of $\frac{1 - \alpha}{2}$ at $(-\gamma,
-1)$. Figure~\ref{fig:example-bis} illustrates this distribution.
We first examine the expected hinge loss $\hinge(w)$ of an arbitrary
linear classifier $w \in \Rset$ in dimension one:
\begin{align*}
  \hinge(w)
  & = \frac{\alpha}{2} \bracket*{\max \curl*{0, 1 + wr} + \max \curl*{0, 1 + wr}}
  + \frac{1 - \alpha}{2} \bracket*{\max \curl*{0, 1 - w\gamma}
    + \max \curl*{0, 1 - w\gamma}}\\
  & = \alpha \max \curl*{0, 1 + wr} + (1 - \alpha) \max \curl*{0, 1 - w\gamma}.
\end{align*}
Thus, distinguishing cases based on the value of scalar $w$, we can write:
\begin{align*}
\hinge(w)
& =
\begin{cases}
  \alpha (1 + wr) + (1 - \alpha) (1 - w\gamma)
  & \text{if } w \in \bracket*{0, \frac{1}{\gamma}}\\
  \alpha (1 + wr) & \text{if } w \geq \frac{1}{\gamma}\\
  \alpha (1 + wr) + (1 - \alpha) (1 - w\gamma)
  & \text{if } w \in \bracket*{-\frac{1}{r}, 0}\\
    (1 - \alpha) (1 - w\gamma) & \text{if } w \leq -\frac{1}{r}.
\end{cases}\\
& =
\begin{cases}
  w (\alpha r - (1 - \alpha) \gamma) + 1
  & \text{if } w \in \bracket*{0, \frac{1}{\gamma}}\\
  \alpha (1 + wr) & \text{if } w \geq \frac{1}{\gamma}\\
  w (\alpha r - (1 - \alpha) \gamma) + 1
  & \text{if } w \in \bracket*{-\frac{1}{r}, 0}\\
    (1 - \alpha) (1 - w\gamma) & \text{if } w \leq -\frac{1}{r}.
\end{cases}
\end{align*}
To simplify the discussion, we will set $r = 1$ and
$\alpha = \frac{\gamma}{2}$, with $\gamma \ll 1$. This, implies
$(\alpha r - (1 - \alpha) \gamma) = \alpha (-1 + 2\alpha) < 0$.  As a
result of this negative sign, the best solution for the first two
cases above is $w = \frac{1}{\gamma}$, $w = 0$ in the third case, and
$w = -\frac{1}{r}$ in the last case. The loss achieved in the two
latter cases is $1$ and $(1 - \alpha) (1 + \frac{\gamma}{r})$, both
larger than the loss $\alpha \paren[\big]{1 + \frac{r}{\gamma}}$ obtained
in the first two cases.  In view of that, the overall minimizer of
$\hinge(w)$ is given by $w^* = \frac{1}{\gamma}$, with
$\hinge(w^*) = \frac{1}{2} (1 + \gamma)$. Note that the zero-one loss
of $w^*$ is $\ell(w^*) = \alpha = \frac{\gamma}{2}$.

Thus, for this example, the norm of the hinge-loss minimizer is
arbitrary large: $\norm{w^*} = \frac{1}{\gamma} \gg 1$. In particular,
for a sample size $m$, we could choose $\gamma < \frac{1}{m}$, leading
to $\norm{w^*} > m$.  Note that, here, any other positive classifier,
$w > 0$, achieves the same zero-one loss as $w^*$. For example, $w =
1$ achieves the same performance as $w^*$ with a more favorable norm.

Our analysis was presented for the population hinge loss but a similar
result holds for the empirical hinge loss.

\begin{figure*}[t]
\centering
\includegraphics[scale=0.2]{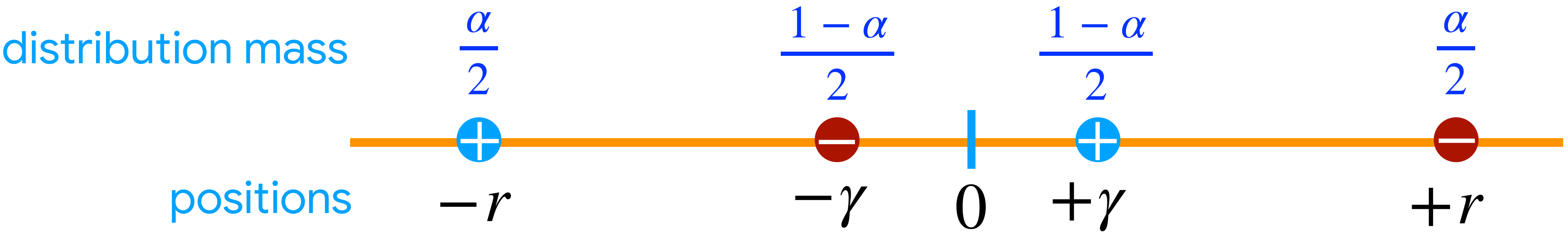}
\caption{Simple example in dimension one for which the minimizer of
  the expected hinge loss $\E[\hinge(w)]$ is $w^* = \frac{1}{\gamma}$
  and thus $\norm{w^*} = \frac{1}{\gamma} \gg 1$ for $\gamma \ll 1$.}
\label{fig:example-bis}
\vskip -.1in
\end{figure*}

\end{document}